\documentclass[twocolumn, final, onepage]{IEEEtran}

\usepackage{amssymb}
\usepackage{latexsym}

\usepackage{url}
\usepackage{xcolor}
\usepackage{graphicx}

\usepackage{amsmath}
\usepackage{amssymb}
\usepackage{amsthm}
\usepackage{mathrsfs}
\usepackage{booktabs}
\usepackage[round, authoryear]{natbib}
\usepackage[all]{xy}
\usepackage{cjhebrew}

\newcommand{\field}[1]{\mathbb{#1}}
\newcommand{\power}{\mathscr{P}}
\newcommand{\set}[1]{\mathcal{#1}}

\newcommand{\nats}{\field{N}}

\newcommand{\define}{:=}
\newcommand{\comment}[1]{}

\newcommand{\sfield}{\set{F}}        
\newcommand{\prob}{\mathsf{P}}      
\newcommand{\expect}{\mathsf{E}}               

\theoremstyle{plain}
\newtheorem{theorem}{Theorem}
\newtheorem{lemma}[theorem]{Lemma}
\newtheorem{proposition}[theorem]{Proposition}
\newtheorem{corollary}[theorem]{Corollary}

\theoremstyle{definition}
\newtheorem{definition}[theorem]{Definition}

\begin{document}

\title{Subjectivity, Bayesianism, and Causality}
\author{Pedro~A.~Ortega \thanks{Date: 24 April 2015---Email: ope@seas.upenn.edu---Address: School of Engineering and Applied Sciences, University of Pennsylvania, 220 S 33rd St, Philadelphia, PA 19104, U.S.A.}}
\date{July 15, 2014}

\maketitle

\begin{abstract}
Bayesian probability theory is one of the most successful frameworks 
to model reasoning under uncertainty. Its defining property is 
the interpretation of probabilities as degrees of belief in 
propositions about the state of the world relative to an
inquiring subject. This essay examines the notion of subjectivity by
drawing parallels between Lacanian theory and Bayesian probability theory,
and concludes that the latter must be enriched with causal interventions
to model agency. The central contribution of this work is an abstract
model of the subject that accommodates causal interventions in a 
measure-theoretic formalisation.
This formalisation is obtained through a game-theoretic \textit{Ansatz} based
on modelling the inside and outside of the subject as an extensive-form game 
with imperfect information between two players. Finally, I illustrate
the expressiveness of this model with an example of causal induction.
\end{abstract}


\section{Introduction}

\IEEEPARstart{E}{arly} modern thinkers of the Enlightenment---spurred by the developments of empirical
science, modern political organisation, and the shift from collective religion to
personal cults---found in the free, autonomous, and rational \emph{subject} the
\textit{locus} on which to ground all of knowledge \citep{Mansfield2000}.
Most notably, Descartes, with his axiom \textit{cogito ergo sum} (`I think, therefore I am'),
put forward the idea that the thought process of the subject is an unquestionable
fact from which all other realities derive---in particular of oneself, and in general
of everything else \citep{Descartes1637}. 

This proposition initiated a long-lasting debate among philosophers such as Rousseau
and Kant, and its discussion played a fundamental r\^{o}le in shaping modern 
Western thought. Indeed, the concept 
of the subject operates at the heart of our core institutions: the legal and political 
organisation rests on the assumption of the free and autonomous subject
for matters of responsibility of action and legitimisation of ruling bodies; 
capitalism, the predominant economic system, depends on forming, through the 
tandem system of education and marketing, subjects that engage in work and 
consumerism \citep{Burkitt2008}; natural sciences equate objective truth 
with inter-subjective experience \citep{Kim2005}; and so forth.

Nowadays, questions about subjectivity are experiencing renewed interest 
from the scientific and technological communities. Recent technological 
advances, such as the availability of massive and ubiquitous computational
capacity, the internet, and improved robotic systems, have triggered the 
proliferation of autonomous systems that monitor, process and deploy 
information at a scale and extension that is unprecedented in history.
Today we have social networks that track user preferences 
and deliver personalised mass media, algorithmic trading systems that account
for a large proportion of the trades at stock exchanges, unmanned vehicles 
that navigate and map unexplored terrain. What are the ``users'' that
a social network aims to model? What does an autonomous system know and what
can it learn? Can an algorithm be held responsible for its actions?
Furthermore, latest progress in neuroscience has both posed novel questions
and revived old ones, ranging from investigating the neural bases of 
perception, learning, and decision making, to understanding the nature 
of free will \citep{Sejnowski2006}. Before these questions can be addressed in a way that is
adequate for the mathematical disciplines, it is necessary
to clarify what is meant by a subject in a way that enables
a quantitative discussion.

The program of this essay is threefold. First, I will argue
that Bayesian probability theory is a subjectivist theory, encoding
many of our implicit cultural assumptions about subjectivity. To support 
this claim, I will show that some basic concepts in Bayesian 
probability theory have a counterpart in Lacanian theory, which
is used in cultural studies as a conceptual framework to structure
the discourse about subjectivity. In the second part, I will put 
forward the claim that Bayesian probability theory needs to be enriched 
with causal interventions to model agency. Finally, I will consolidate 
the ideas on subjectivity in an abstract mathematical synthesis. The 
main contribution of this formalisation is the measure-theoretic
generalisation of causal interventions.

\section{Subjectivity in Lacanian Theory}

To artificial intelligence, statistics, and economics, the 
questions about subjectivity are not novel at all: many can be
traced back to the early discussions at the beginning of the
twentieth century that eventually laid down the very foundations 
of these fields. Naturally, these ideas did not spring out of a
vacuum, but followed the general trends and paradigms of the time.
In particular, many of the fundamental concepts about subjectivity
seem to have emerged from interdisciplinary cross-fertilisation.

For instance, in the humanities, several theories of subjectivity were
proposed. These can be roughly subdivided into two dominant approaches 
\citep{Mansfield2000}: the \emph{subjectivist}/\emph{psychoanalytic} 
theories, mainly associated with Freud and Lacan, which see the subject 
as a \emph{thing} that can be conceptualised and studied 
\citep[see \textit{e.g.}\ ][]{Freud1899, Fink1996}; and the \emph{anti-subjectivist}
theories, mainly associated with the works of Nietzsche and Foucault, 
which regard any attempt at defining the subject as a \emph{tool} of 
social control, product of the culture and power of the time 
\citep{Nietzsche1887, Foucault1964}.

For our discussion, it is particularly useful to investigate
the relation to Lacan\footnote{It shall be noted however, that 
Lacan's work is notoriously difficult to understand, partly due 
to the complexity and constant revisions of his ideas, but most 
importantly due to his dense, multi-layered, and obscure prose 
style. As a result, the interpretation presented
here is based on my own reading of it, which was significantly
influenced by \citet{Fink1996}, \citet{Mansfield2000} and the
work by \v{Z}i\v{z}ek \citep{Zizek1992, Zizek2009}.}, firstly 
because it is a subjectivist theory and secondly because its
abstract nature facilitates establishing the relation to 
Bayesian probability theory. Some ideas that are especially 
relevant are the following.

\emph{The subject is a construct.} There is a consensus among
theorists (both subjectivist and anti-subjectivists) that the subject 
is not born into the world as a unified entity. Instead, her
constitution as a unit is progressively built as she experiences
the world \citep{Mansfield2000}. The specifics of this unity
vary across the different accounts, but roughly speaking,
they all take on the form of an acquired sense of separation
between a \textit{self} (inside) and the rest of the world (outside).
For instance, during the early
stages of their lives, children have to learn that their limbs belong
to them. In Lacan for instance, this distinction is embodied in the 
terms \emph{I} and \emph{the Other} (Fig.~\ref{fig:psychoanalytic-subject}a). 
Crucially, Lacan stresses that the subject is precisely this ``membrane''
between inward and outward flow \citep{Fink1996}.

\begin{figure}[htbp]
\centering %
\includegraphics[width=\columnwidth]{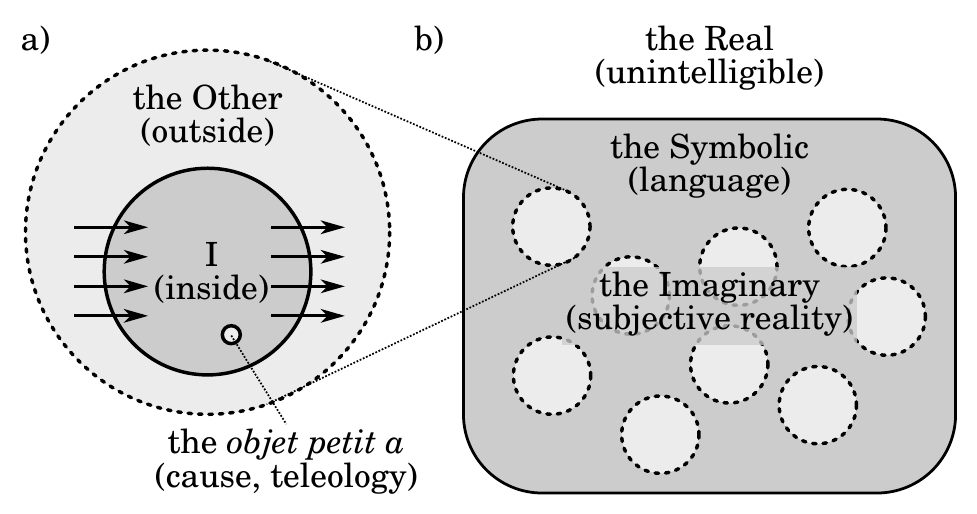}
\caption{The subject in Lacanian theory.}
\label{fig:psychoanalytic-subject} %
\end{figure}

\emph{The subject is split.} Structurally, the subject is divided
into a part that holds beliefs about the world, and a part 
that governs the organisation and dynamics of those beliefs
\emph{in an automatic fashion}. The most well-known instantiation of this idea
is the Freudian distinction between the \emph{conscious} 
and the \emph{unconscious}, where the latter constitutes
psychological material that is repressed, but nevertheless accessible
through dreams and involuntary manifestations such as a 
``slip of the tongue'' \citep{Freud1899}. Here however, the interpretation
that is more pertinent to our analysis is Lacan's. In
his terminology, the two aforementioned parts correspond 
to the \emph{imaginary} and the \emph{symbolic} registers respectively 
(Fig.~\ref{fig:psychoanalytic-subject}b). 
Simply put,  the imaginary can be described as the collection of concepts 
or images that, when pieced together, make up the totality of 
the subject's ontology: in particular, the world and the subject's 
sense of self. In other words, the imaginary register is 
responsible for entertaining hypotheses about reality. In turn, 
these images are organised by the symbolic register into a network 
of meaning that is pre-given, static, and ``structured like a language'' 
\citep{Lacan1953}.

\emph{Language is a system of signification.} Many of the modern
ideas about knowledge and subjectivity are centred around language. 
In this view, the subject is seen as a signifying entity that produces 
and consumes signs (linguistic material) in the form of spoken 
language, images, and general sensorimotor expression \citep{Saussure1916}.
Language then can be thought of as a system of signs that operates
by detecting signifiers (labels) and associating them to signifieds
(meanings or ideas)---possibly in cascade, with the signifieds being
the signifiers of later stages. Crucially, the associations between
signifiers and signifieds are arbitrary and contingent, established
by pure convention (think of `apple', `manzana', `mela', `Apfel',
`pomme', `\textcjheb{.hwpt}', \textit{etc.}). The influence of these views is witnessed 
by the adoption of related ideas by thinkers from fields ranging 
from logic \citep{Russell1905, Wittgenstein1921},
philosophy of language \citep{Wittgenstein1953}, phenomenology
\citep{Heidegger1927}, rhetoric \citep{Knape2000}, and 
linguistics/cognitivism \citep{Chomsky1957} to computer 
science \citep{Turing1936} and biology/cybernetics 
\citep{Maturana1970, Maturana1987}.

\emph{The real is the engine of the subject.} The imaginary and the
symbolic registers refer to the subject's intellect, that is, to the 
organisation of the things that she can potentially comprehend or experience,
and their structure is static. There is a third register in Lacan's 
conceptualisation, namely the \emph{real} (Fig.~\ref{fig:psychoanalytic-subject}b),
representing the unintelligible, random source of external perturbations 
that the subject picks up and integrates into her symbolic domain 
in the form of sense-data, thereby setting her knownledge in motion 
(compare \textit{e.g.}\ to the ``web of beliefs'' of \citet{Quine1951}).

\emph{Teleology.} Finally, there is the question of purposeful behaviour.
In Lacan, teleology (see~Fig.~\ref{fig:psychoanalytic-subject}a) is related to 
what he calls the \textit{objet petit a}: that is, an unexpected incoherence
that interrupts the otherwise regular chain of signification 
\citep{Lacan1973, Zizek1992}. Such an interruption has two consequences 
that are worth pointing out. First, the deviation from the regular 
chain of signification can be thought of as an expression
of spontaneous desire, \textit{i.e.}\ a sudden jerk that steers
the chain into different, \emph{preferred} consequences. Second, the interrupted
signifying chain, by injecting randomness, introduces an independence
of choice that entails a responsibility, a claim to ownership of cause, 
and a post-rationalisation of the subject's decisions. In short: 
a detected irregularity signals \emph{agency}. For instance, in the
sequence
\[
1,\quad 2,\quad 3,\quad 4,\quad 5,\quad
6,\quad 8,\quad 9,\quad 10,
\]
the missing number~$7$ breaks the pattern and can give the impression
that it was intentionally omitted.

\section{Subjectivity in Bayesian Probability Theory}\label{sec:subj-bayesian}

In the mathematical disciplines, one of the most prominent theories
dealing with subjectivity is Bayesian probability theory. Its current
formal incarnation came to be as a synthesis of many fields such 
as measure theory \citep[see \textit{e.g.}\ ][]{Lebesgue1902, Kolmogorov1933},
set theory \citep{Cantor1874}, and logic \citep{Frege1892, Russell1905,
Wittgenstein1921}. After Bayes' and Laplace's initial epistemic
usage of probabilities\nocite{Bayes1763, Laplace1774}, founders
of modern probability theory have \emph{explicitly} started
using probabilities as degrees of subjective belief. On one hand, they
have postulated that subjective probabilities can be inferred by observing
actions that reflect personal beliefs \citep{Ramsey1931, 
deFinetti1937, Savage1954}; on the other hand, they
regarded probabilities as extensions to logic under epistemic
limitations \citep{Cox1961, Jaynes2003}. Importantly, both
accounts rely on a subject that does statistics in the world
having belief updates governed by Bayes' rule. 

Bayesian probability theory, in its capacity as a subjectivist theory,
can be related to ideas in Lacanian theory. Recall that formally, probability 
theory provides axiomatic foundations for modelling experiments\footnote{Here, 
I will use the word \textit{experiment} in a very broad sense, 
including the thought processes of the subject throughout her entire life.} 
involving randomness. Such a \textit{randomised experiment}
takes the form of a probability space $(\Omega, \sfield, \prob)$, where $\Omega$ is a
set of possible states of nature, $\sfield$ is a $\sigma$-algebra on $\Omega$ 
(a collection of subsets of $\Omega$ that is closed under countably many set operations,
comprising complement, union, and intersection), and $\prob$ is a probability measure over $\sfield$.
Given this setup, I suggest the following correspondences, summarised in 
Tab.~\ref{tab:correspondences}:
\begin{enumerate}
\item \emph{Real $\leftrightarrow$ generative/true distribution.}
In probability theory, it is assumed that there exists
a source that secretly picks the state of Nature $\omega \in \Omega$
that is then progressively ``revealed'' through measurements.
Some measure theory textbooks even allude to the irrational, unintelligible
quality of the source\footnote{Note that this allusion goes at least as far back
as Hesiod's \textit{Theogony} dating from the pre-philosophical era. The \textit{Theogony}
describes the creation of the world by the \textit{muses}---a literary device
used at the time standing for something that is unintelligible.} 
by using the phrase ``Tyche, the goddess 
of chance, picks a sample'' to describe this choice 
\citep[see for instance][]{Billingsley1978, Williams1991}. 

\item \emph{Symbolic $\leftrightarrow$ probability space.}
Conceptually, the $\sigma$-algebra $\sfield$ of a probability space
contains the universe of all the yes/no questions (\textit{i.e.}\ propositions) that the subject
can entertain. A particular aspect of a given state of Nature
$\omega$ is extracted via a corresponding random variable 
$X:\Omega \rightarrow \mathcal{X}$, mapping $\omega$ into a
symbol $X(\omega)$ from a set of symbols $\mathcal{X}$.
Random variables can be combined to form complex aspects,
and the ensuing symbols are consistent (\textit{i.e.}\ free of contradictions)
as guaranteed by construction. Thus, a probability space and 
the associated collection of random variables make up the 
structure of the potential realities that the subject can 
hope to comprehend. Furthermore, one can associate to each 
random variable at least one of three r\^{o}les (but typically
just one), detailed next.

\item \emph{Imaginary $\leftrightarrow$ hypotheses.}
A random variable can play the r\^{o}le of a \emph{latent feature} of the
state of Nature. Latent variables furnish the sensorimotor space 
with a conceptual or signifying structure, and a particular
configuration of these variables constitutes a hypothesis
in the Bayesian sense. Because of this function, we can 
associate the collection of latent variables to Lacan's imaginary
register.

\item \emph{Flow between I and the Other $\leftrightarrow$ actions \& observations.}
The hypotheses by themselves do not ground the subject's symbolic
domain to any reality however---for this, variables modelling interactions
are required. These variables capture symbols that appear in the
sensorimotor stream of the subject, that is, at her boundary
with the world, modelling the directed symbolic flow occurring between
the I and the Other; in particular, the out- and inward flows are represented
by actions and observations respectively.

\item \emph{\textit{Objet petit a} $\leftrightarrow$ causal intervention.}
The last connection I would like to establish, which will become a central
theme in what follows, is between the \textit{object petit a} and causal
interventions. Lacanian theory explains agency in terms of a kink
in the signifying chain---that is, the interruption of a pre-existing relation
between two symbols---that is subjectivised \emph{in hindsight} \citep{Fink1996,
Zizek1992}. One crucial aspect of this notion is that it requires the 
comparison between two instants of the signifying network, namely the one 
where the relation is still intact and the resulting one where the relation 
is absent, adding a \emph{dynamic} element to the static symbolic order. 
This element has \emph{no} analogue in standard probability theory. 
However, the last twenty years have witnessed the systematic study of 
what appears to be an analogous idea in the context of probabilistic causality. 
More precisely, the interruption of the signifying chain is a causal intervention
\citep{Pearl2009}.
\end{enumerate}

\begin{table}[htbp]
\caption{Lacanian and Bayesian theories of the subject.} %
\label{tab:correspondences} %
\bigskip %
\centering %
\begin{tabular}{ccc}
  \toprule
  Lacan & & Bayes \\
  \midrule
  real (register) & $\longleftrightarrow$ & true distribution \\
  symbolic (register) & $\longleftrightarrow$ & probability space \\
  imaginary (register) & $\longleftrightarrow$ & hypotheses \\
  the Other $\rightarrow$ I (flow) & $\longleftrightarrow$ & observations \\
  I $\rightarrow$ the Other (flow) & $\longleftrightarrow$ & actions \\
  \textit{objet petit a} & $\longleftrightarrow$ & causal intervention \\
  \bottomrule
\end{tabular}
\end{table}

One can establish a few more connections, for instance between Lacan's concept
of \textit{jouissance} and the economic term \textit{utility}, but I hope that
the  aforementioned ones suffice to make my case for now.

In summary, my claim is that Bayesian probability theory is almost an axiomatic
subjectivist theory; ``almost'' because it lacks an analogue of the function 
performed by the \textit{objet petit a}, namely causal interventions, which 
is crucial to fully characterise the subject. This will be the goal of the next section.

\section{Causality and the Abstract Subject}

Thus, Bayesian probability theory can be taken as a mathematical theory of the subject 
that \emph{passively} draws logical and probabilistic inferences from experience.
To extend it to \emph{interactive} subjects, \textit{i.e.}\ subjects that can shape the course 
of realisation of a randomised experiment, it is necessary to introduce
additional formal machinery. Due to space limitations, the thorough
analysis of these requirements is deferred to Appendix~\ref{sec:probabilistic-causality}, 
and the resulting synthesis into a measure-theoretic model of the interactive
subject is presented in Appendix~\ref{sec:abstract-subject}. My goal here is
to give an informal summary of the main ideas and results found therein.

\subsection{Causality}

Causality has always been one of the central aspects 
of human explanation, with its first philosophical discussion dating back
to Aristotle's \textit{Physics} roughly some 2500 years ago. In spite
of this, it has not received much attention from the scientific community, 
partly due to the strong scepticism expressed by \citet{Hume1748} and later 
by prominent figures in statistics \citep{Pearson1899} and logic \citep{Russell1913}. 
It is only in the recent decades that philosophers and computer scientists 
have attempted to characterise causal knowledge in a rigorous way 
\citep{Suppes1970, Salmon1980, Rubin1974, Cartwright1983, Spirtes2001, Pearl2009, 
Woodward2001, Shafer1996, Dawid2007}. I refer the reader to \citet{Dawid2010}
for a mathematical, and \citet{IllariRusso2014} for a thorough philosophical comparison
between existing approaches.

Arguably, one of the central contributions has been Pearl's characterisation of
causal interventions \citep{Pearl1993, Pearl2009}, which in turn draws ideas
from \citet{Simon1977} and \citet{Spirtes2001}. Informally, a causal
intervention is conceived as a manipulation of the probability law
of a random experiment that functions by holding the value 
of a chosen random variable fixed. The operation has been formalised
in causal directed acyclic graphs (DAGs); structural models \citep{Pearl2009};
chain graphs \citep{Lauritzen2002}; influence diagrams \citep{Dawid2007}
and decision problems \citep{Dawid2015}, to mention some. Another approach 
that is worth pointing out is that of \citet{Shafer1996}. Therein, 
Shafer shows that simple probability trees are able to capture
very rich causal structures, although he does not define causal 
interventions on them. While the aforementioned 
definitions differ in their scope, interpretation, and degree of complexity, 
ultimately they entail transformations on probability measures that 
are mathematically consistent with each other.

\subsection{Causality in Games with Imperfect Information}

Interestingly though, one the earliest and most general formalisations
of what much later became known as causal interventions comes from game theory. 
More precisely, the characteristics underlying modern causal interventions 
feature implicitly in the representation of \textit{extensive-form games 
with imperfect information} \citep{Neumann1944,Osborne1999}. This connection
is straightforward but, rather surprisingly, rarely acknowledged 
in the literature. The game-theoretic \emph{Ansatz} yields an elegant 
definition that can express a rich set of causal dependencies, including
higher-order ones. Conceptually, the approach is similar to defining Pearl-type 
interventions on Shafer's probability trees, and it lends itself 
to a measure-theoretic formalisation.

In this set-up, the subject interprets her experience
as a sequential game between two players named \emph{I} and \emph{W} standing
for the I and the Other (\textit{i.e.}\ the World) respectively\footnote{Multiple ``Others''
are folded into a single player \emph{W}.}---that is, player \emph{I}'s moves
are the subject's actions, while \emph{W}'s moves are the subject's observations. 
The structure of the game embodies the causal narrative of the 
randomised experiment. Crucially, player \emph{W}'s otherness is expressed by
a special ability: it can play \emph{secret moves} that 
\emph{I} knows about but cannot see\footnote{Game theory here 
subtly assumes that players have a theory of mind, that is to say,
the presumption that other players possess private information or 
inaccessible ``mental'' states that direct their behaviour.}. 
As a consequence, causal assumptions manifest themselves 
as independence  relations that strictly limit player \emph{I}'s 
knowledge about player \emph{W}'s moves. The missing information 
has the effect of rendering some states in this game indistinguishable 
from each other in the eyes of player \emph{I}\footnote{In game-theoretic jargon, these states form
so-called \emph{information sets}.}. In contrast, player \emph{W}
is omniscient, being able to see all the moves in the game.

To further clarify the peculiarities of this rather unusual
set-up, it is helpful to make an analogy to dreams. In many of them, it is often
the case that we encounter other people to whom we talk. 
These conversations can be vivid and---most importantly---surprising. 
Yet, how come we cannot anticipate what others will say, given that these dreams
are entirely orchestrated by our own imagination? Somehow, our dreams 
go to great lengths to maintain this stable illusion of otherness 
of the people and the world by hiding their motives and reasons from us. 
That is, our imagination draws a demarcation line between the things 
done by ourselves (which do not surprise us) and the things done by others 
(which feel external and novel). To Lacan, this is no different
when we are awake. The imaginary register is responsible for maintaining
this illusion of the self and the world (players \emph{I} 
and \emph{W} respectively), and it does so within the confines of 
our language (\textit{i.e.}\ the symbolic, providing the rules of the game).

From a syntactic point of view, perhaps the two most important 
findings of the analysis in Appendix~\ref{sec:probabilistic-causality} are that, 
firstly, an interactive subject must distinguish between her 
\emph{actions} and \emph{observations} because they entail different 
belief updates, and secondly, that in order to do so, the subject 
must differentiate between an \emph{event}
(\textit{i.e.}\ a logical proposition about the experiment) and a \emph{realisation}
(\textit{i.e.}\ a possible state of the experiment). I refer the reader
to Appendix~\ref{sec:probabilistic-causality} for the details of this argument.

\subsection{Representation of Causal Dependencies}

To illustrate the set-up, consider the following classical example (with a twist), 
consisting of two random variables W and B representing the weather 
and the atmospheric pressure as measured by a mercury barometer. 
To simplify, assume that both W and B are binary, taking values in 
$\mathcal{W}=\{$sunny, rainy$\}$ and $\mathcal{B}=\{$low, high$\}$
respectively. From elementary physics, we know that the 
pairs $($sunny, low$)$ and $($rainy, high$)$ are more
likely than $($sunny, high$)$ and $($rainy, low$)$. 

Our subject (\textit{e.g.}\ an extraterrestrial visitor) does not know 
whether the weather controls the barometer's mercury
column or the other way around. Using the language of causal
DAGs, this situation can be represented as one where the subject
ponders the plausibility of the two competing causal hypotheses, say
H $= 1$ and H $= 2$, depicted in Fig.~\ref{fig:hypotheses}a, 
and the subject is challenged to \emph{induce} one from experience.
Crucially though, while each hypothesis can be represented as a 
separate graphical model, the causal dependencies governing the 
induction problem itself \emph{cannot} be cast into a single 
causal DAG. Since the direction of the arc between W 
and B is unknown, the subject has to treat H as a latent variable
that has to be inferred. Thus, we can tie the two hypotheses
into a single diagram by adding a node for the latent variable
as shown in Fig.~\ref{fig:hypotheses}b. However, the new diagram
is \emph{not} a causal DAG any more. Most importantly, 
the theory of causal DAGs does not specify how to do inference 
on this type of graph\footnote{Note that this problem persists when
using the language of \textit{structural equations} \citep{Pearl2009}, 
as the hypothesis~H controls whether to include either the structural
equation B $= g($W$, $U$)$ or W $= h($B$, $V$)$ into the model, 
but at the same time, inferences are only defined once all the 
structural equations of the system are in place.} \citep{Ortega2011b}.

\begin{figure}[htbp]
\centering %
\includegraphics[width=0.6\columnwidth]{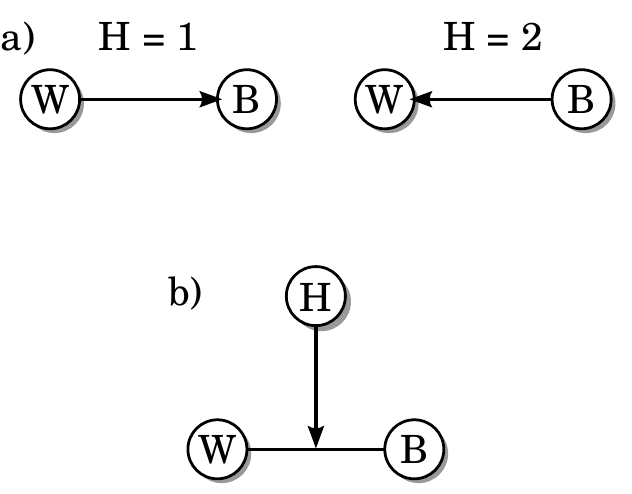}
\caption{Combination of causal hypotheses.}\label{fig:hypotheses} 
\end{figure}

In contrast, this situation has a natural representation as a 
probability tree. This is because
a probability tree allows specifying causal dependencies that
are dynamically instantiated as a result of previous events.
We first observe that the hypothesis~H 
causally precedes the weather~W and the height of the mercury
column~B, as~H determines the very causal order of~W and~B. 
The resulting probability tree is depicted in Fig.~\ref{fig:barometer}. 
The semantics in a probability tree are self-explanatory: nodes represent
the potential states of realisation of the random experiment,
and the arrows indicate the possible transitions between them,
which are taken with the probabilities indicated in the transition
labels. Importantly, the arrows encode causal dependencies,
in that they specify the order in which the intermediate states 
of the experiment are determined. Hence, a path starting at the root 
and ending in a leaf corresponds to a full realisation of the 
random experiment. The probability of a particular realisation
(indicated by a bold number below the leaf) is obtained by multiplying
the probabilities of all the transitions taken.

\begin{figure}[htbp]
\centering %
\includegraphics[width=\columnwidth]{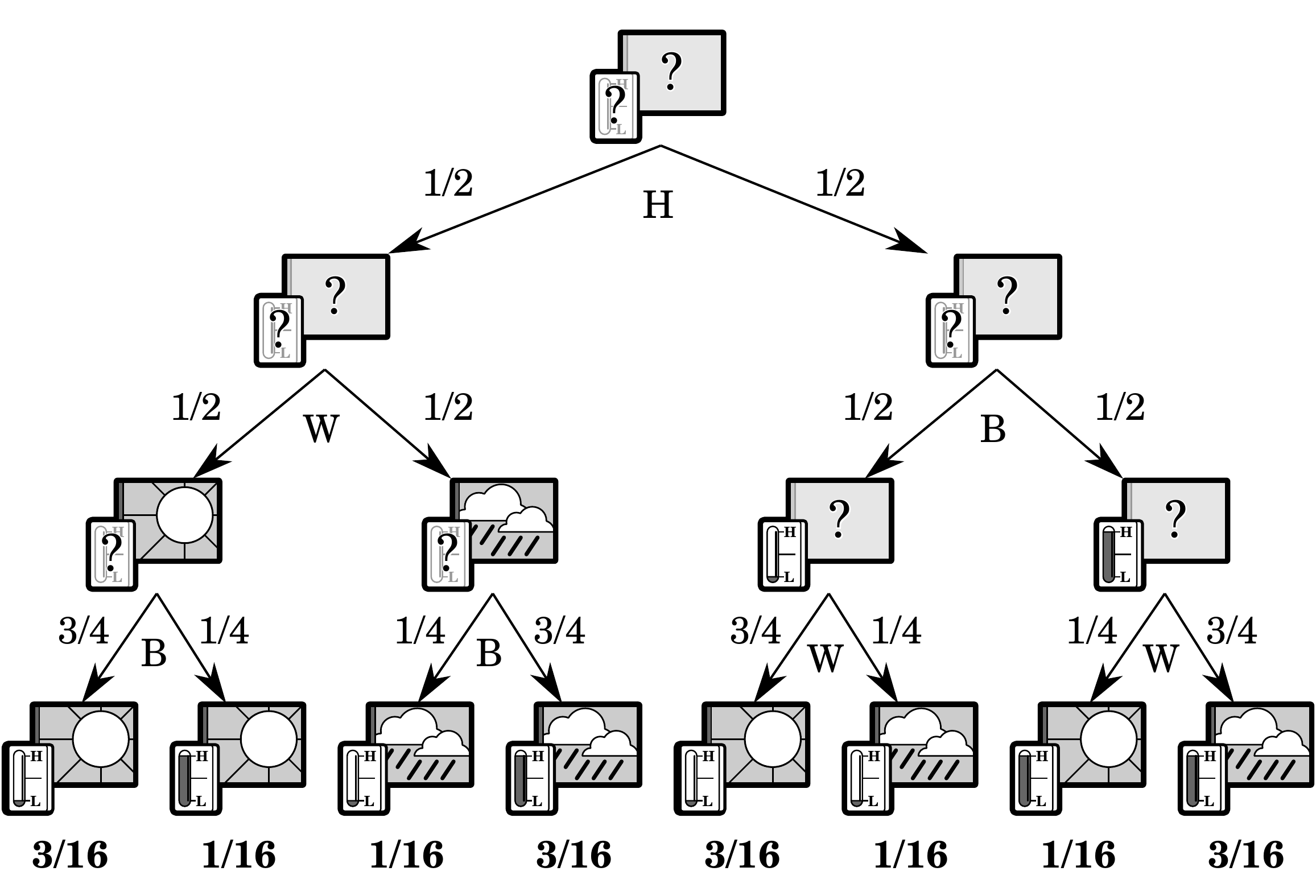}
\caption{Probability tree for the induction problem.}
\label{fig:barometer} %
\end{figure}

The probability tree in Fig.~\ref{fig:barometer} also shows that the subject 
holds uniform prior beliefs $\prob($H$)=1/2$ over the two hypotheses as 
encoded by the two transitions at the root node. This inclusion of subjective
probabilities is possible because we commit to a fully Bayesian interpretation 
of the probabilities in the tree. Finally, note that the left and right
subtrees from the root node encode both a causal order and a likelihood
function. In particular, the likelihood function $\prob($W, B$|$H$)$
is identical for both causal hypotheses, and consequently there is no way 
to distinguish between them through purely observational data---that is, 
in a game where every single move is decided by player~\emph{W}.

Formally, a probability tree can be characterised through a \emph{causal space}
consisting of: a \emph{sample space}; a collection of privileged nested events
called the \emph{set of realisations}; and a \emph{causal probability measure}
assigning transition probabilities (precise definitions are provided in Appendix~\ref{sec:abstract-subject}). 
Such a causal space contains enough information to generate the $\sigma$-algebra 
and the probability measure of a classical probability space. For instance, 
Fig.~\ref{fig:barometer-causal-space} illustrates a causal space suitable 
for modelling the probability tree of Fig.~\ref{fig:barometer}, where the
states of realisation of the experiment have been labelled $S_0, S_1, \ldots, S_{15}$.

\begin{figure}[htbp]
\centering %
\includegraphics[width=\columnwidth]{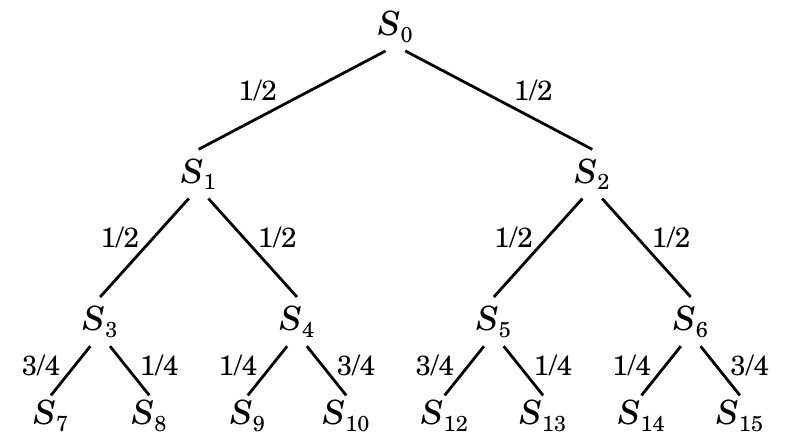}
\caption{Causal space of the induction problem.}
\label{fig:barometer-causal-space} %
\end{figure}

\subsection{Actions and Observations}

As mentioned previously, actions and observations correspond to moves taken
by players~\emph{I} and~\emph{W} respectively. During the course of the game, 
the subject keeps track of the probabilistic state of knowledge \emph{as seen 
from the perspective of player~I}. In particular, the moves of player~\emph{I} 
cannot depend on information that is private to player~\emph{W}. 

\paragraph{Observations} First, suppose that the subject makes her
first interaction by \emph{observing} that the barometer's mercury column
has fallen to a low value, \textit{i.e.}\ B = low. This can happen under three possible moves by player~\emph{W}, namely $S_3 \rightarrow S_7$, $S_4 \rightarrow S_9$, 
and $S_2 \rightarrow S_5$. The observation allows the subject to rule out 
all the realisations of the experiment that take any of the incompatible
transitions $S_3 \rightarrow S_8$, $S_4 \rightarrow S_{10}$ or $S_2 \rightarrow S_6$. 
Each one of the three compatible transitions postulates an alternative
underlying state of the game. For instance, the move $S_3 \rightarrow S_7$ 
presumes that player~\emph{W} had secretly played the moves $S_0 \rightarrow S_1$ 
and $S_1 \rightarrow S_3$ before eventually revealing the last move to player~\emph{I}. 
However, since the latent state of the game is unknown to player~\emph{I}, 
the subject updates her beliefs by conditioning on B = low. Using Bayes' rule, 
the posterior probability of the causal hypothesis H = 1 is equal to
\[
  \prob(\text{H} = 1|\text{B} = \text{low}) 
  = \frac{ \prob(\text{B} = \text{low}|\text{H} = 1) \prob(\text{H} = 1)}
    { \sum_h \prob(\text{B} = \text{low}|\text{H} = h) \prob(\text{H} = h) }
  = \frac{1}{2},
\]
where the likelihood $\prob(\text{B} = \text{low}|\text{H} = 1)$ is 
calculated by marginalising over the weather
\[
  \sum_{w} \prob(\text{W} = w|\text{H} = 1) 
  \prob(\text{B} = \text{low}|\text{H} = 1, \text{W} = w) = \frac{1}{2},
\]
and where $\prob(\text{B} = \text{low}|\text{H} = 2) = 1/2$ is as specified in 
the transition $S_2 \rightarrow S_5$. The posterior probability of the causal
hypothesis H = 2 can then be obtained through the sum rule of probabilities:
\[
  \prob(\text{H} = 2|\text{B} = \text{low}) 
  = 1 - \prob(\text{H} = 1|\text{B} = \text{low})
  = \frac{1}{2}.
\]
Since $\prob($H$|$B$ = $low$) = \prob($H$)$, the observation of the
barometer does not provide the subject with evidence favouring 
any of the two causal hypotheses.
 
\paragraph{Actions} Second, assume that the subject \emph{acts}
by setting the value B = low herself instead. In this case, it is player~\emph{I}
that makes one of the three aforementioned moves rather than player~\emph{W}. 
But, as is seen in Fig.~\ref{fig:barometer-causal-space}, these moves 
possess different probabilities. In order to draw the action
according to the prescribed odds, player~\emph{I} would have to know the underlying 
state of the game. Given that this is not the case, the subject must correct 
the probabilities of the randomised experiment before conditioning on 
B = low to account for player~\emph{I}'s 
ignorance. Technically, this correction is done as follows:
\begin{enumerate}
  \item[i.] \emph{Critical Bifurcations.} 
  The subject identifies all the states of realisation having at least
  one transition leading to \emph{any} compatible state and at least 
  one transition leading \emph{only} to incompatible states. These states
  are called \emph{critical bifurcations}. In the case of setting B = low, 
  the critical bifurcations are given by the states $S_2$, $S_3$, and~$S_4$. 
  \item[ii.] \emph{Transition Probabilities.} 
  For each critical bifurcation, the subject assigns zero probability 
  to every transition leading only to incompatible states, thereafter
  renormalising the remaining transitions. 
  Fig.~\ref{fig:barometer-causal-space-intervened} shows the 
  resulting causal space.
\end{enumerate}
This operation is analogous to Pearl's notion of \emph{causal 
intervention}, and the reader can find its mathematically rigorous definition
in Appendix~\ref{sec:abstract-subject}. From a conceptual point of view however, Pearl's 
motivation was to characterize the notion of \emph{manipulation}, while here 
it arises as an accounting device for hiding information from player \emph{I}. 

After this correction has been made, 
a new causal probability measure $\prob_\text{B=low}$ is obtained, and the subject 
can condition her beliefs on B = low as if it were a normal
observation. Note that $\prob_\text{B=low}$ preserves all the transition probabilities
of $\prob$ save for the intervened ones. 

Thus, the posterior probability of the causal hypothesis H~=~1 given that
player \emph{I} set B = low is equal to
\begin{align*}
  \prob_{\text{B=low}}&(\text{H} = 1|\text{B} = \text{low})
  \\&= \frac{ \prob_\text{B=low}(\text{B} = \text{low}|\text{H} = 1) \prob_\text{B=low}(\text{H} = 1)}
    { \sum_h \prob_\text{B=low}(\text{B} = \text{low}|\text{H} = h) \prob_\text{B=low}(\text{H} = h) }
  = \frac{1}{2},
\end{align*}
where, similarly to the observational case, the likelihood 
$\prob_\text{B=low}(\text{B} = \text{low}|\text{H} = 1)$ is calculated by 
marginalising over the weather, but now relative to the intervened distribution
\[
  \sum_{w} \prob_\text{B=low}(\text{W} = w|\text{H} = 1) 
  \prob_\text{B=low}(\text{B} = \text{low}|\text{H} = 1, \text{W} = w) = 1,
\]
and where the likelihood of the second hypothesis is
\[
  \prob_\text{B=low}(\text{B} = \text{low}|\text{H} = 2) = 1.
\]

\begin{figure}[htbp]
\centering %
\includegraphics[width=\columnwidth]{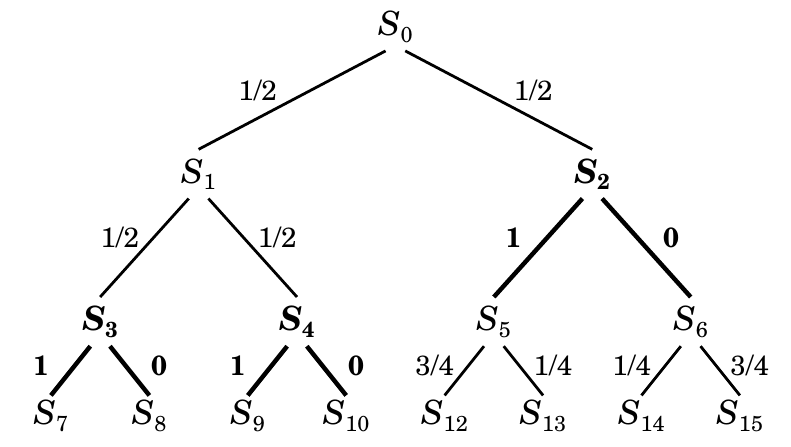}
\caption{Causal space after intervention.}
\label{fig:barometer-causal-space-intervened} %
\end{figure}

Fig.~\ref{fig:observation-vs-intervention} schematically illustrates
the difference between the previous action and observation in terms
of how the probability mass is adjusted during each belief update. Each 
rectangle illustrates the nested arrangement of the states of 
realisation (vertical axis) and their corresponding
probability masses (horizontal axis). Belief updates are carried out in two steps:
elimination of the incompatible probability mass and a subsequent
renormalisation.

\begin{figure}[htbp]
\centering %
\includegraphics[width=\columnwidth]{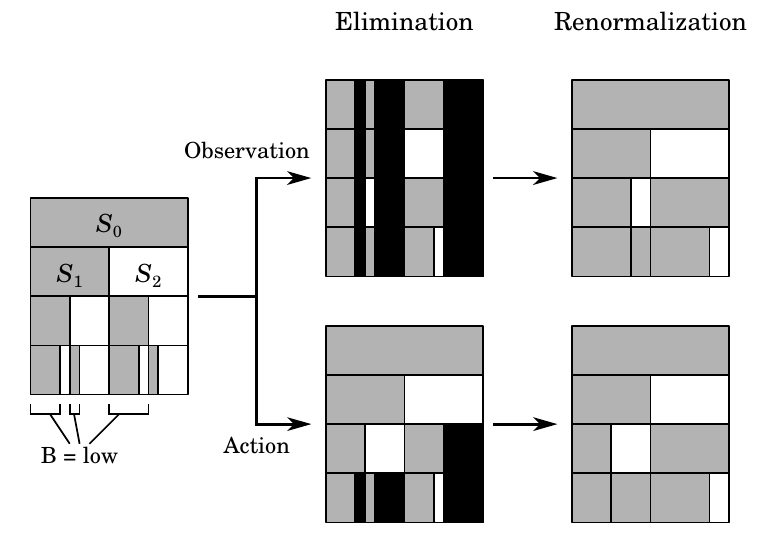}
\caption{Observations versus interventions.}
\label{fig:observation-vs-intervention} %
\end{figure}

\subsection{Posterior}

We now conclude our example by calculating the posterior distribution over the
two causal hypotheses. Note that both belief updates leave the posterior
distribution unchanged, \textit{i.e.}\
\[
  \prob(\text{H}|\text{B} = \text{low})
  = \prob_\text{B=low}(\text{H}|\text{B} = \text{low})
  = \prob(\text{H}),
\]
but they do so for different reasons: in the case of the observation, 
because the two causal hypotheses are observationally indistinguishable; 
and in the case of the action, because the choice of the hypothesis H
precedes the barometer reading B. 

Indeed, assume that the subject subsequently sees that the weather is 
rainy. Following an analogous calculation as before, it can be shown that
\[
  \prob(\text{H} = 1|\text{B} = \text{low}, \text{W} = \text{rainy})
  = \frac{1}{2},
\]
for the purely observational case, and
\[
  \prob_\text{B=low}(\text{H} = 1|\text{B} = \text{low}, \text{W} = \text{rainy})
  = \frac{2}{3},
\]
for the intervened case. Thus, the subject, by virtue of her action, can
render the two causal hypotheses observationally distinguishable, entitling her
to conclude that it is more likely that the weather controls the barometer just 
by looking at the weather. However, she does not rule out the alternative.

\section{Discussion}

\subsection{Causal Spaces}

The model of the subject advocated here is in no way intended to
be a replacement for existing causal frameworks. Rather, it aims at 
providing a common abstract ground containing a minimal
set of primitives to reason about general causal interventions. This was achieved 
by supplying the $\sigma$-algebra of standard probability theory with a set of 
realisations encoding the causal dependencies governing its events.

The model detailed in Appendix~\ref{sec:abstract-subject} is limited to 
countable sets of realisations. This was chosen so as to ensure 
that the $\sigma$-algebras generated by realisation sets are always 
well defined (\textit{e.g.}\ do not have excessive cardinalities). Furthermore,
the axioms guarantee that causal interventions are always well defined
for any given event.

\subsection{Dynamic Instantiation of Causes}

Causal spaces can express causal dependencies that are instantiated
dynamically, that is to say, causes that only come into force under suitable
higher-order causal conditions. Not only are these dynamic causes ubiquitous 
in Nature and society (\textit{e.g}\ genes that regulate the \textit{relation} 
between other genes; smoking causing cancer leads governments to restrict 
tobacco; \textit{etc.}), but also they are necessary to represent 
causal induction problems (see previous section).

\subsection{R\^{o}le of Interventions}
To stress the importance of the function performed by causal interventions
we can compare it to a mechanism found in the human immune system. The physical 
distinction between us and our surrounding seems obvious to us (\textit{e.g.}\ we
speak of our \textit{body}) but from a biomolecular point of view it is far from 
clear where these boundaries are. Much to the contrary: these boundaries 
must be actively maintained by \textit{antibodies}, which identify and
sometimes neutralise \textit{antigens} (\textit{i.e.}\ foreign bodies, 
such as bacteria and viruses) for their subsequent 
removal by other parts of the immune system.

Similarly, causal interventions tag those events that are attributed
to the subject's self to distinguish them from those generated by the world. 
Without them, the subject would be devoid of the psychological apparatus that
gives her a sense of unity and agency, and she would experience her life as 
a ``film with no recognizable protagonist'' that she could identity herself with.

\subsection{What is an Action?}

Classical decision theory assumes from the outset that there is a clear-cut
distinction between action and observation variables \citep{Neumann1944, 
Savage1954}. Similarly, control theory and artificial intelligence
take this distinction as a given. Indeed, artificial intelligence
textbooks describe an agent as any system possessing sensors
and effectors \citep{RussellNorvig2009}.

In contrast, the idea of the subject as a construct plus the line
of thought developed here suggest a different story. Whenever the
subject sees an event as emanating from herself, she tags it through
a causal intervention. However, every causal intervention changes her
beliefs in an irreversible manner. For instance, consider the subject's
choice probabilities of fixing the value of the barometer herself
from the induction example. These were equal to
\[
  \prob(\text{B} = \text{low}) = \frac{1}{2}
  \quad\text{and}\quad
  \prob(\text{B} =\text{ high}) = \frac{1}{2}
\]
before the intervention, and
\[
  \prob_\text{B=low}(\text{B} = \text{ low}) = 1
  \quad\text{and}\quad
  \prob_\text{B=low}(\text{B} =\text{ high}) = 0
\]
immediately thereafter but \emph{prior} to conditioning her beliefs 
on B $=$ low. In other words, she went from being completely undecided to being 
absolutely convinced that this had been her choice all 
along\footnote{Also, refer to the discussion in Appendix~\ref{sec:sequential-realisations}.}.
Plus, the new probabilibity measure $\prob_\text{B=low}$ does not possess enough information 
to recover the original probability measure $\prob$, yet it does contain the 
traces left by the intervention in the form of independences from the causal 
precedents. Accordingly, 
a random variable can be identified as an action if it can be thought of as 
the result of a causal intervention. But then again, how can she tell that 
the independences in $\prob_\text{B=low}$ were not already there from before the 
supposed intervention?

These observations raise the following basic open questions:
\begin{enumerate}
 \item What is the criterion employed by the subject to decide whether to
 treat an event as an action or an observation? To what extent is this arbitrary?
 \item Does this criterion have to be learned in the form of an hypothesis
 in the Bayesian sense?
 \item Does a subject gain anything from distinguishing between herself and the world?
\end{enumerate}
Indeed, if the subject wanted to learn what her actions and observations are, 
then the only way she can classify a random variable unambiguously, as is the 
case in classical decision theory, is when she has no hypothesis offering 
a competing explanation.

This also sheds more light into the connection between causal interventions
and Lacan's \textit{objet petit a}\footnote{The term \textit{objet petit a}
loosely translates into ``object little other''. The ``\textit{a}'' in 
\textit{objet petit a} stands for the French word \textit{autre} (other).}. 
An action turns out to be a random variable of somewhat contradictory 
nature: because on one hand, it is statistically independent and hence, 
subjectivised; but on the other hand, it is still generated by the
external world, namely, by the very decision processes of the subject
that are not under its direct control, \textit{e.g.}\ following a utility-maximising agenda.
Lacan's term \textit{objet petit a} can thus be regarded as a play of words
that encapsulates this dual nature \citep{Fink1996}.

\subsection{Concluding Remarks}

There are numerous reasons why I chose to compare Bayesian probability 
theory to Lacanian theory. It is true that, virtually since its inception, 
psychoanalytic theories have always faced fierce opposition that 
has questioned their status as a scientific discipline 
\citep[see \textit{e.g.}\ ][]{Popper1959, Feynman1964}. 
While their efficacy as a treatment of mental illnesses is undoubtedly controversial 
\citep{Tallis1996}, cultural studies have embraced psychoanalytic theories
as effective conceptual frameworks to structure the discourse about
subjectivity in a metaphysically frugal fashion. As a researcher in artificial 
intelligence, the greatest value I see in the psychoanalytic
theories is in that they epitomise the contingent cultural assumptions 
about subjectivity of modern Western thinking, summarising ideas
that otherwise would require a prohibitive literature research.

Finally, I would like to stress that, while here my motivation 
was to advance a mathematical definition of subjectivity (in Bayesian terms), 
the resulting axiomatic system is agnostic about its interpretation. The reader 
can verify that many of the existing philosophical views are indeed compatible
with the mathematical formalisation presented here \citep{IllariRusso2014}, 
and that some of the philosophical interpretations are not unrelated to the 
psychoanalytic interpretation (\textit{e.g.}\ the idea of 
\textit{agency probabilities} put forward by \citet{MenziesPrice1993} and the 
\textit{epistemic} interpretation of causality of \citet{Williamson2009}).

\section{Acknowledgements}

The writing of this essay has benefited from numerous conversations 
with D.A.~Braun, D.~Balduzzi, J.R.~Donoso, A.~Saulton, E.~Wong and
R.~Nativio. Furthermore, I wish to thank the anonymous reviewers
for their invaluable suggestions on how to improve this manuscript; 
M.~Hutter and J.P.~Cunningham
for their comments on a previous version of the abstract model of
causality (in my doctoral thesis); and A.~Jori, Z.~Domotor, 
and A.P.~Dawid for their insightful lectures and discussions on 
ancient philosophy, philosophy of science
and statistical causality at the Universities of T{\"u}bingen and Pennsylvania.
This study was funded by the Ministerio de Planificaci{\'o}n de Chile
(MIDEPLAN); the Emmy Noether Grant \mbox{BR 4164/1-1} (“Computational and 
Biological Principles of Sensorimotor Learning”); and by grants 
from the U.S. National Science Foundation, Office of Naval Research 
and Department of Transportation.

\appendices

\section{Probabilistic Causality}\label{sec:probabilistic-causality}

\subsection{Evidential versus Generative Probabilities}

One of the features of modern probability theory is that
a probability space $(\Omega, \sfield, \prob)$ can be used
in two ways, which we may label as \emph{evidential} and \emph{generative}.
The evidential use conceives a probability space as a representation
of an experimenter's knowledge about the conclusions he can draw when he is
provided with measurements performed on the outcomes; while the
generative usage sees the probability space as a faithful
characterisation of the experiment's stochastic mechanisms that bring about
observable quantities. Thus, a statistician or a philosopher can use
probabilities to asses the \emph{plausibility} of hypotheses, while an engineer
or a physicist typically uses them to characterise the \emph{propensity} of an
event, often assuming that these propensities are objective, physical properties
thereof.

In a Bayesian interpretation, evidential and generative
refer to subject's observations/measurements and actions/choices respectively.
Under the evidential usage of probabilities, the subject 
passively contemplates the measurements of phenomena generated by the
world. A measurement reveals to the subject which possible worlds she 
can discard from her knowledge state. In contrast, under the 
generative usage of probabilities, the subject \emph{is} the 
random process itself. Thus, outcomes are chosen randomly 
by the subject and then communicated to the world. 
While there are many cases where this distinction does not
play a r\^{o}le, if we aim at characterising a subject 
that both passively observes and actively chooses, this distinction
becomes crucial.

Our running example consists of a three-stage experiment involving 
two identical urns: urn A containing one white and three black 
balls, and urn B having three white and one black ball. 
In stage one, the two urns are either swapped or not with uniform probabilities.
In stage two it is randomly decided whether to
exclude the left or the right urn from the experiment. If the urns have not
been swapped in the first stage, then the odds are $3/4$ and $1/4$ for keeping
the left and the right urn respectively. If the urns have been swapped, then
the odds are reversed. In the third and last stage, a ball is drawn from the urn
with equal probabilities and its colour is revealed.  We associate each stage
with a binary random variable: namely Swap $\in$ $\{$yes, no$\}$, Pick $\in$
$\{$left, right$\}$ and Colour $\in$ $\{$white, black$\}$ respectively.
Figure~\ref{fig:experiment} illustrates the set-up. In calculations, I will
sometimes abbreviate variable names and their values with their first letters.
We will now consider several interaction protocols between two players named \textit{I}
and \textit{W}, representing the outward and inward flow of a subject
respectively as detailed in Section~\ref{sec:subj-bayesian}.

\begin{figure}[htbp]
\centering %
\includegraphics[width=\columnwidth]{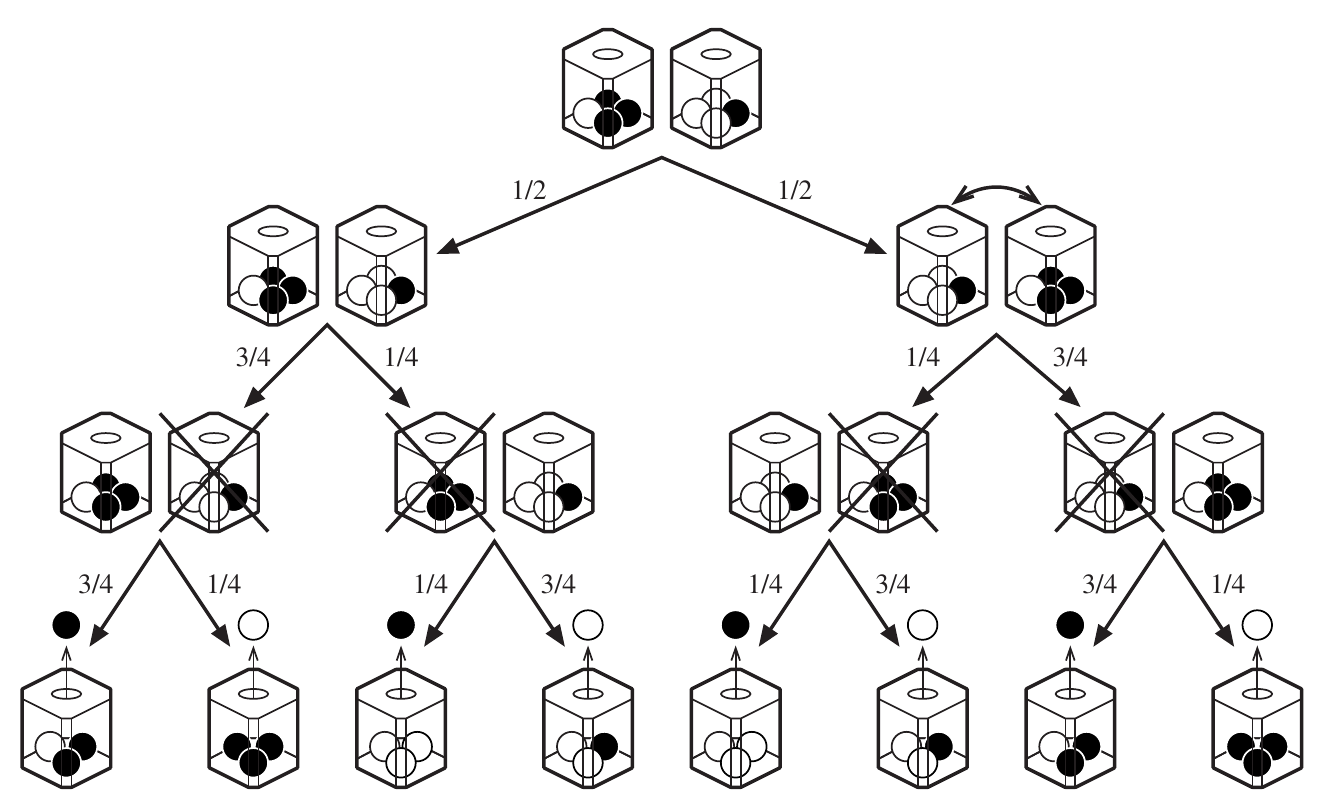}
\caption{A three-stage randomised experiment.}
\label{fig:experiment} %
\end{figure}

\subsubsection{Generative}

In the generative case, \textit{I} carries out the three steps of the experiment,
possibly consulting auxiliary randomising devices like tosses of fair coins. 
In each step, \textit{I} makes a random \emph{choice}; that is, it selects
a value for the corresponding random variable following a prescribed 
probability law that depends on the previous choices. For instance, 
the odds of ``drawing a black ball'' given that ``the urns have been 
swapped in the first stage and the right urn has been picked in the 
second stage'' is $3/4$.

The probabilities governing \textit{I}'s behaviour are formalised with a
probability space $S \define (\Omega,\sfield,\prob)$, where $\Omega_1$
contains the eight possible outcomes, $\sigma$-algebra is the powerset
$\sfield = \power(\Omega_1)$, and $\prob$ is the probability measure that
is consistent with the conditional probabilities in Figure~\ref{fig:experiment}.
Table~\ref{tab:experiment} lists the eight outcomes and their probabilities.

\begin{table}[htbp]
\caption{Outcome probabilities in probability space $S$} %
\label{tab:experiment} %
\bigskip %
\centering %
\begin{tabular}{cccc}
  \toprule
  Swap & Pick & Colour & Probability \\
  \midrule
  no  & left  & black & 9/32 \\
  no  & left  & white & 3/32 \\
  no  & right & black & 1/32 \\
  no  & right & white & 3/32 \\
  yes & left  & black & 1/32 \\
  yes & left  & white & 3/32 \\
  yes & right & black & 9/32 \\
  yes & right & white & 3/32 \\
  \bottomrule
\end{tabular}
\end{table}

The information contained in the probability space does not enforce a
particular sequential plan to generate the outcome. The story of the experiment
tells us that Swap, Pick, and Colour are chosen in this order. However, \textit{I} 
can construct other sequential plans to generate the outcome. 
For example, \textit{I} could first choose the value of Colour, then
Swap, and finally Pick (possibly having to change the underlying story about
urns and balls), following probabilities that are in perfect accordance with the
generative law specified by the probability space.

\subsubsection{Evidential}

In this case, player \textit{I}, knowing about the probability law governing 
the experiment, passively observes its realisation as chosen by \textit{W}. 
In each step, it makes a \emph{measurement};
that is, \textit{I} obtains the value of a random variable and uses it to update its
beliefs about the state of the outcome. For instance, the plausibility of ``the
ball is black'' given that ``the urns have been swapped in the first stage and
the right urn has been picked in the second stage'' is~$3/4$.

Here again, the probabilities governing \textit{I}'s beliefs are formalised by
the same probability space $S$. Analogously to the generative case,
it does not matter in which order the information about the outcome is revealed
to \textit{I}: for instance, $\prob(\text{Colour}|\text{Swap, Pick})$ is the
same no matter whether it observes the value of Swap or Pick first.

\subsection{Mixing Generative and Evidential}

Let us change the experimental paradigm. Instead of letting~\textit{W} choosing
and~\textit{I} passively observing, we now let both determine the
outcome, taking turns in steering the course of the experiment depicted 
in Figure~\ref{fig:experiment}. In the first stage, \textit{W} chooses between
swapping the urns or not; in stage two, \textit{I} decides randomly whether to keep the
left or the right urn; and in the last stage, \textit{W} draws a ball from the remaining
urn. The protocol is summarised in Table~\ref{tab:protocol}. We will investigate
two experimental conditions.

\begin{table}[htbp]
\caption{Protocol for the experiment} %
\label{tab:protocol} %
\bigskip %
\centering %
\begin{tabular}{ccc}
  \toprule
  Stage & Variable & Chosen by \\
  \midrule
  1  & Swap  & \textit{W} \\
  2  & Pick  & \textit{I} \\
  3  & Colour & \textit{W} \\
  \bottomrule
\end{tabular}
\end{table}

\subsubsection{Perfect Information.}

Under the first condition, both players are fully aware of all the previous
choices. At any stage, the player-in-turn makes a decision
following the conditional probability table that is consistent with
past choices. It is easy to see that, again, the probability space $S$
serves as a characterisation of the subject: although this time the conditional
probabilities stand for \textit{I}'s beliefs (first and last stage) 
and \textit{I}'s behaviour (second stage). Essentially, the fact that 
we now have interactions between \textit{I} and \textit{W} can still 
be dealt with under the familiar analytical framework of
probability theory. Note that, as in the previous two cases, we can suggest
changes to the sequential order; furthermore, we can swap the players'
r\^{o}les without changing our calculations.

\subsubsection{Imperfect Information}

The second experimental regime is identical to the previous one with one exception:
\textit{W} carries out the first stage of the experiment secretly, \emph{without
telling} \textit{I} whether the urns were swapped or not. Hence, for \textit{I}, the
statements ``Swap = yes'' and ``Swap = no'' are equiprobable (\textit{e.g.}\ the urns
are opaque, see Fig.~\ref{fig:transparent-opaque}). How should \textit{I}
choose in this case? Let us explore two attempts.

\begin{figure}[htbp]
\centering %
\includegraphics[width=5cm]{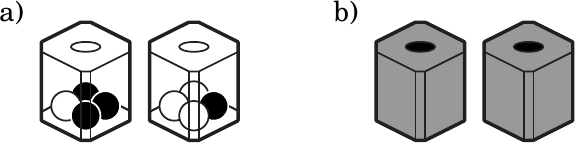}
\caption{Transparent versus opaque.}
\label{fig:transparent-opaque} %
\end{figure}

The first attempt consists in postulating that the two experimental
regimes (perfect and imperfect information) are decoupled and hence 
require case-based probability specifications.
Concretely, $\prob(\text{Pick}|\text{Swap $=$ yes})$,
$\prob(\text{Pick}|\text{Swap $=$ no})$ and $\prob(\text{Pick})$ are
unrelated probability distributions and are therefore freely specifiable.
While this is a possible solution, it has the drawback that the resulting
belief model violates the probability axioms, since in general the equation
\[
  \prob(\text{P}) \neq
      \sum_{\text{S $=$ y, n}}
          \prob(\text{P}|\text{S}) \, \prob(\text{S})
\]
does not hold for tuples of conditional probability distributions.

The second attempt enlarges the model as follows. We add an auxiliary binary
variable, say KnowsSwap, that indicates whether \textit{I} is in possession of
the value of variable Swap. This allows specifying a total of four conditional
probability distributions of the form
\[
  \prob(\text{Pick}|\text{Swap}, \text{KnowsSwap})
\]
indexed by the joint settings of Swap and KnowsSwap, where the latter can be
treated as another random variable or as a conditional variable. However, this
arrangement does not fundamentally bypass the original problem: we
can extend \textit{I}'s ignorance of the value of Swap to the value of
KnowsSwap as well. That is, although we have extended Pick's functional
dependency from Swap to both Swap and KnowsSwap, \emph{there is no reason
why KnowsSwap should not be an undisclosed variable too}. Consequently, this
would require introducing yet another auxiliary variable, say KnowsKnowsSwap, to
indicate whether the value of KnowsSwap is known, and so forth, piling up an
infinite tower of indicator variables. Eventually, one is left with the
feeling that this second solution is conceptually unsatisfactory as well.

Thus, let us continue with a solution that accepts \textit{I}'s
ignorance of the value of the random variable Swap. 
The story of the experiment tells us that the probabilities
$\prob(\text{Pick}|\text{Swap})$ have the semantics of conditional instructions
for~\textit{I}; but since the condition is unknown, the
choice probabilities consistent with this situation are obtained
by marginalising over the unknown information. More specifically, 
the probability of picking the right urn is
\begin{align*}
    \prob(\text{P $=$ r})
    &= \sum_{\text{S $=$ y,n}} \prob(\text{P $=$ r}|\text{S}) \,
\prob(\text{S}),
\end{align*}
which is thereby rendered independent of the unknown information. For the
particular numbers in our example, the choice probabilities evaluate to a uniform
distribution
\[
  \prob(\text{P $=$ r})
    = \frac{3}{4} \cdot \frac{1}{2} + \frac{1}{4} \cdot \frac{1}{2}
    = \frac{1}{2},
   \qquad \prob(\text{P $=$ l}) = \frac{1}{2}.
\]
Interestingly, the resulting experiment does not follow the same generative 
law as in the previous experimental
condition any more, for the odds of swapping the urns in the first stage and picking the
right urn in the second were $\frac{1}{2} \cdot \frac{3}{4} = \frac{3}{8}$ and
not $\frac{1}{2} \cdot \frac{1}{2} = \frac{1}{4}$ like in the current set-up.
Thus, albeit
\textit{I}'s beliefs are captured by the probability space $S = (\Omega,
\sfield, \prob)$, the outcomes of this new experiment follow a different
generative law, described by a probability triple $S' \define (\Omega,
\sfield, \prob')$, where $\prob' \neq \prob$ is determined by the
probabilities listed in Table~\ref{tab:experiment-2}. The choice made by
\textit{I} actually changed the probability law of the experiment! 

\begin{table}[htbp]
\caption{Outcome probabilities in probability space $S'$.} %
\label{tab:experiment-2} %
\bigskip %
\centering %
\begin{tabular}{cccc}
  \toprule
  Swap & Pick & Colour & Probability \\
  \midrule
  no  & left  & black & 3/16 \\
  no  & left  & white & 1/16 \\
  no  & right & black & 1/16 \\
  no  & right & white & 3/16 \\
  yes & left  & black & 1/16 \\
  yes & left  & white & 3/16 \\
  yes & right & black & 3/16 \\
  yes & right & white & 1/16 \\
  \bottomrule
\end{tabular}
\end{table}

A moment of thought reveals that this change happened simply because 
\textit{I}'s state of knowledge did not conform to the functional 
requirements of the second random variable. At first seemingly harmless,
this change of the probability law has far-reaching consequences for
\textit{I}'s state of knowledge: the familiar operation of probabilistic 
conditioning does not yield the correct belief update any more. 
To give a concrete example, recall that the plausibility of the 
urns having been swapped in the first stage \emph{before} \textit{I}
picks the left urn in the second stage is
\[
    \prob(\text{S $=$ y}) = \frac{1}{2}.
\]
However, \emph{after} the choice, the probability is
\begin{align*}
    \prob(\text{S $=$ y}|\text{P $=$ l})
    &= \frac{ \prob(\text{P $=$ l}|\text{S $=$ y}) \prob(\text{S $=$ y}) }
    { \sum_{\text{S $=$ y, n}} \prob(\text{P $=$ l}|\text{S}) \prob(\text{S}) }
    \\&= \frac{ \frac{1}{4} \cdot \frac{1}{2} }
           { \frac{1}{4} \cdot \frac{1}{2} + \frac{3}{4} \cdot \frac{1}{2} }
    = \frac{1}{4}.
\end{align*}
Hence, if \textit{I} wanted to use probabilistic conditioning to infer 
the plausibility of the hypothesis, then it would conclude that its
choice actually \emph{created evidence} regarding the first stage of 
the experiment---a conclusion that violates common sense.

\subsection{Causal Realisations}\label{sec:sequential-realisations}

If we accept that probabilistic conditioning is not the correct belief update
in the context of generative probabilities then we need to re-examine the
nature of probabilistic choices.

The familiar way of conceptualising the realisation of a random experiment
$(\Omega, \sfield, \prob)$ is via the choice of a sample $\omega \in \Omega$
following the generative law specified by the probability measure
$\prob$. Sequential observations are modelled as sequential refinements (\textit{i.e.}\ a
filtration)
\[
    \sfield_\text{I}
    \xrightarrow{\text{S}} \sfield_\text{II}
    \xrightarrow{\text{P}} \sfield_\text{III}
    \xrightarrow{\text{C}} \sfield
\]
of an initial, ignorant algebra $\sfield_\text{I} = \{\Omega, \varnothing\}$ up
to the most fine-grained algebra $\sfield = \power(\Omega)$. The labels on the
arrows indicate the particular random variable that has become observable (\textit{i.e.}\
measurable) in the refinement. A second, non-standard way in the context of
probability theory, is to think of a realisation as a \emph{random
transformation} of an initial
probability triple $(\Omega, \sfield, \prob)$ into a final, degenerate
probability triple $(\Omega, \sfield, \prob_\omega)$, where $\prob_\omega$ is
the probability measure concentrating all its probability mass on the
singleton $\{\omega\} \in \sfield$. This alternative way of accounting for
random realisations will prove particularly fruitful to formalise probabilistic
choices.

In many situations it is natural to subdivide the realisation of a complex
experiment into a sequence of realisations of simple sub-experiments. 
For instance, the realisation of the experiment in the running 
example can be broken down into a succession of three random choices, 
\textit{i.e.}\ a sequence
\[
    \prob
    \xrightarrow{f_\text{S}} \prob_\text{I}
    \xrightarrow{f_\text{P}} \prob_\text{II}
    \xrightarrow{f_\text{C}} \prob_\text{III},
\]
of random transformations of the initial probability measure $\prob$. Here,
the three mappings $f_\text{S}$, $f_\text{P}$, and $f_\text{C}$ implement
particular assignments for the values of Swap, Pick, and Colour respectively, and
$\prob_\text{I}$, $\prob_\text{II}$, and $\prob_\text{III}$ are their
corresponding resulting probability measures. Together, $f_\text{S}$,
$f_\text{P}$, and $f_\text{C}$ form a \emph{sequential plan} to specify a
particular realisation $\{\omega\} \in \sfield$ of the experiment.
However, the mathematical formalisation of such decompositions requires
further analysis.

Underlying any purely evidential usage of probabilities, there is the implicit,
although somewhat concealed, assumption of a predetermined outcome: the
choice of the outcome of the experiment \emph{precedes} the measurements
performed on it\footnote{Here we recall the probabilistic mythology, in which it is
\emph{Tyche, the Goddess of Chance}, who has the privilege of choosing the
outcome.}. In other words, obtaining information about the outcome updates the
belief state of the subject, but not the outcome itself. In contrast, the
generative use assumes an undetermined, fluid state of the outcome. More
specifically, a choice updates both the beliefs of the subject \emph{and the
very state of the realisation}. Hence, there are two types of states
that have to be distinguished: the state of the beliefs and the state of the
realisation.

Distinguishing between the states of the realisation imposes restrictions on how
beliefs have to be updated after making choices. These restrictions are
probably best highlighted if one imagines---just for illustrative purposes---that the
experiment is a physical system made up from a cascade of (stochastic)
mechanisms, where each mechanism is a sub-experiment implementing a choice. 
Based on the physical metaphor, one concludes that choices can only affect
the odds of the states of realisation that are \textit{downstream}. So, for instance, picking the 
left urn knowing that the urns were swapped in the first stage increases the odds for drawing a white ball in the last stage, but it cannot change the fact that the urns have 
been swapped. Hence, the belief update following a choice affects the 
beliefs about the future, but not about the past\footnote{To be more precise, 
by the terms \textit{past} and \textit{future} I mean the causal precedents 
and causal successors respectively.}.

\begin{figure}[htbp]
\centering %
\includegraphics[width=\columnwidth]{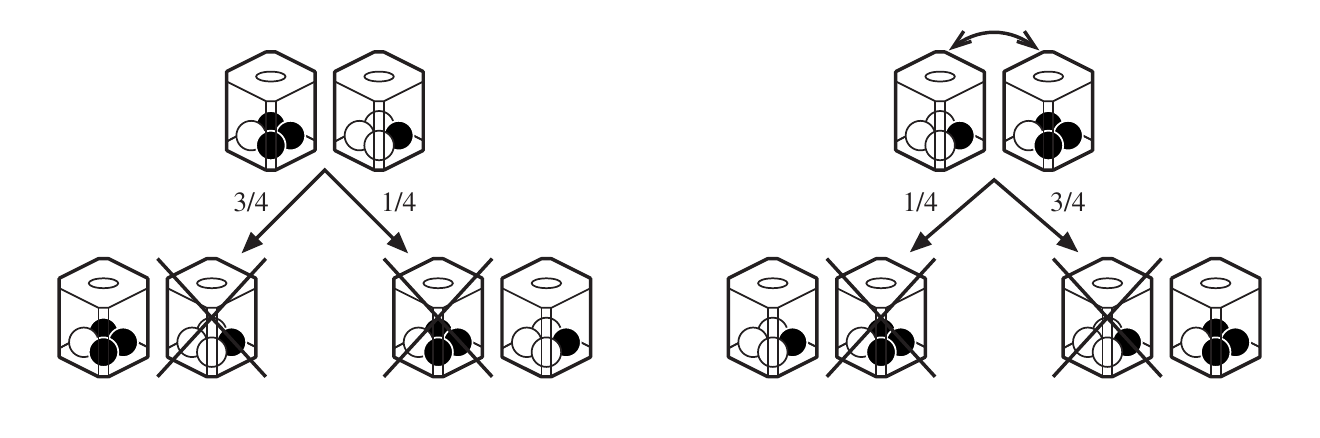}
\caption{The second stage of the randomised experiment.}
\label{fig:stage-two} %
\end{figure}

Another aspect of choices is concerned with their scope, which spans many
potential realisations. Consider the second stage of the experiment. As it can
be seen from its illustration in Figure~\ref{fig:stage-two}, this stage contains
two parts, namely one for each possible outcome of the first stage. In a sense,
its two constituents could be conceived as being actually two stand-alone
experiments deserving to be treated separately, since they represent
alternative, mutually exclusive historical evolutions of the sequential experiment
which were rendered causally independent by the choice in the first stage.
Thus, picking an urn given that the urns were swapped in the first stage could
in principle have nothing to do with picking an urn given that the urns were
not swapped. However, this separation requires knowing the choices in the execution of the
sequential experiment; a situation that failed to happen in our previous example.
To be more precise: the semantics of this grouping is
precisely that we have declared not being able to discern between them.
In game theory, this is called an \textit{information set}---see Appendix~\ref{sec:game-theory} later in the text.

\begin{figure}[htbp]
\centering %
\includegraphics[width=\columnwidth]{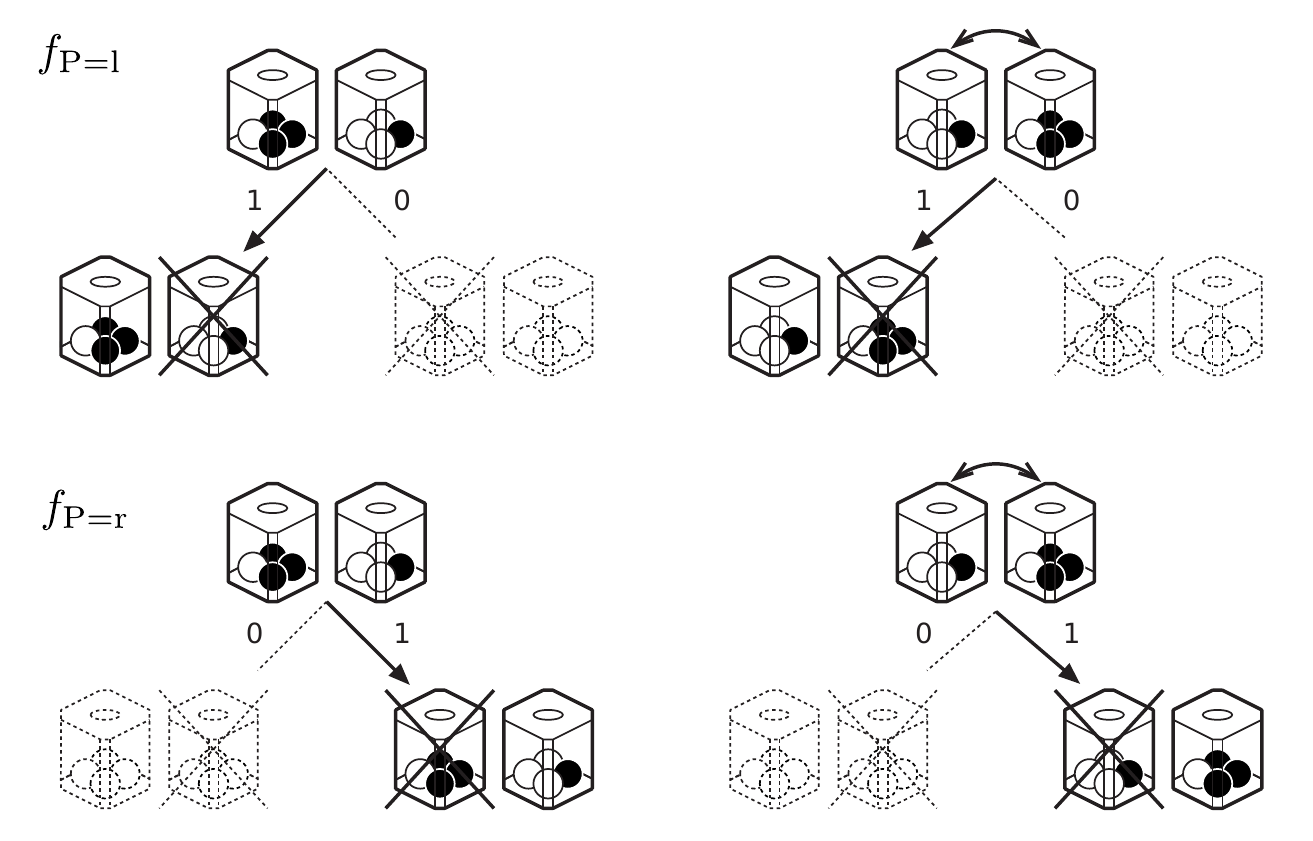}
\caption{The two possible choices in the second stage.}
\label{fig:choices-stage-two} %
\end{figure}

Even though it is clear that the belief update for a choice has to respect the
causal boundaries separating the different histories, a choice is an epistemic operation
that affects \emph{all} histories in parallel because they are the \emph{same} from
the subject's point of view. Therefore, we formalise the sub-experiment 
in Figure~\ref{fig:stage-two} as a
\emph{collection} of experiments admitting two choices, namely $f_\text{P$=$l}$
and $f_\text{P$=$r}$, representing the choice of the left and right urn
respectively. Suppose we are asked to choose the left urn in this collection of
experiments. This makes sense because ``choosing the left urn'' is an operation
that is well-defined across all the members in the collection. Then, this choice
amounts to a transformation that puts all the probability mass on both left
urns. Analogously, ``choosing the right urn'' puts all the probability mass on
the right urns. The two choices, $f_\text{P$=$l}$ and $f_\text{P$=$r}$, are
illustrated in Figure~\ref{fig:choices-stage-two}.

\begin{figure}[htbp]
\centering %
\includegraphics[width=\columnwidth]{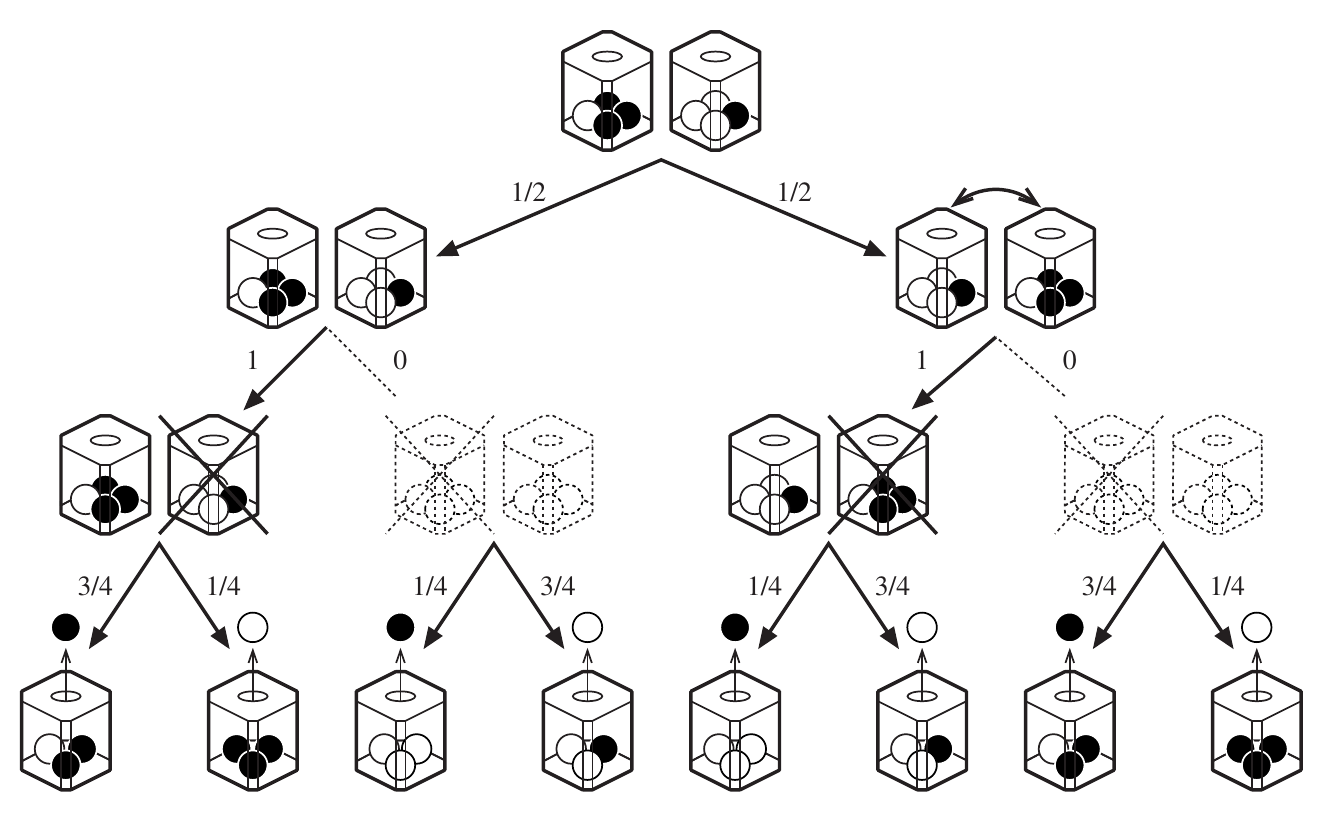}
\caption{The three-stage randomised experiment after choosing the left urn in
the second stage.}
\label{fig:manipulated-experiment} %
\end{figure}

Recall the situation where player \textit{I} chose the left urn without knowing
whether the urns were swapped or not in the first stage. From our previous
discussion, this amounts to applying $f_\text{P$=$l}$ to the sub-experiment in
the second stage. This leads to the modified three-stage experiment illustrated
in Figure~\ref{fig:manipulated-experiment} having a probability measure
$\prob_\text{P$=$l}$, where the subscript informs us of the manipulation
performed on the original measure $\prob$. Similarly, if $f_\text{P$=$r}$ is
applied, we obtain the probability measure $\prob_\text{P$=$r}$ for the
experiment. Table~\ref{tab:manipulated-experiment} lists both probability
measures plus the expected probability measure $\expect[\prob_\text{P}]$ which
averages over the two choices. Notice that in Table~\ref{tab:experiment-2},
it is seen that $\expect[\prob_\text{P}]$ is equal to $\prob'$, \textit{i.e.}\ the
probability law resulting from the experiment under the condition of imperfect
information.

\begin{table}[htbp]
\caption{Probabilities of the experiment after the choice of Pick.} %
\label{tab:manipulated-experiment} %
\bigskip %
\centering %
\begin{tabular}{cccccc}
  \toprule
  Swap & Pick & Colour & $\prob_\text{P=L}$ & $\prob_\text{P=R}$ & $\expect[ \prob_\text{P} ]$ \\
  \midrule
  no  & left  & black & 3/8 & 0   & 3/16 \\
  no  & left  & white & 1/8 & 0   & 1/16 \\
  no  & right & black & 0   & 1/8 & 1/16 \\
  no  & right & white & 0   & 3/8 & 3/16 \\
  yes & left  & black & 1/8 & 0   & 1/16 \\
  yes & left  & white & 3/8 & 0   & 3/16 \\
  yes & right & black & 0   & 3/8 & 3/16 \\
  yes & right & white & 0   & 1/8 & 1/16 \\
  \bottomrule
\end{tabular}
\end{table}

Finally, we calculate the plausibility of the urns having been swapped after
the left urn is chosen. This is given by
\begin{align*}
    \prob_\text{P$=$l}(\text{S $=$ y}|\text{P $=$ l})
    &= \frac{ \prob_\text{P$=$l}(\text{P $=$ l}|\text{S $=$ y})
        \prob_\text{P$=$l}(\text{S $=$ y}) }
    { \sum_{\text{s$=$y,n}} \prob_\text{P$=$l}(\text{P $=$ l}|\text{S $=$ s})
        \prob_\text{P$=$l}(\text{S $=$ s}) }
    \\&= \frac{ 1 \cdot \prob(\text{S $=$ y}) }
    { \sum_{\text{s$=$y,n}} 1 \cdot \prob(\text{S $=$ s}) }
    \\&= \frac{ 1 \cdot \frac{1}{2} }
           { 1 \cdot \frac{1}{2} + 1 \cdot \frac{1}{2} }
    = \frac{1}{2}.
\end{align*}
Hence, according to the belief update for choices that we have proposed in this
section, choosing the left urn in the second stage does not provide evidence
about the first stage. However, if a black ball is drawn right afterwards,
the posterior plausibility will be
\begin{align*}
  \prob_\text{P$=$l}(&\text{S $=$ y}|\text{P $=$ l}, \text{C $=$ b}) \\
    &= \frac{ \prob_\text{P$=$l}(\text{C $=$ b}|\text{S $=$ y}, \text{P $=$ l})
      \prob_\text{P$=$l}(\text{S $=$ y}|\text{P $=$ l}) }
    { \sum_{\text{s$=$y,n}}
      \prob_\text{P$=$l}(\text{C $=$ b}|\text{S $=$ s}, \text{P $=$ l})
      \prob_\text{P$=$l}(\text{S $=$ s}|\text{P $=$ l}) }
    \\&= \frac{ \frac{1}{4} \cdot \frac{1}{2} }
    { \frac{1}{4} \cdot \frac{1}{2} + \frac{3}{4} \cdot \frac{1}{2} }
    = \frac{1}{4},
\end{align*}
\textit{i.e.}\ the subject obtains evidence favouring the hypothesis that the urns
were not swapped. This leads to an interesting interpretation. In a sense, the
intervention functions as a psychological mechanism
informing the subject that her choice cannot be used as additional evidence to
support hypotheses about the past; the fundamental r\^{o}le of the intervention is
to declare the choice as an unequivocal, deterministic consequence of the
subject's state of knowledge at the moment of the decision. Or, loosely
speaking, the intervention ``tricks the subject into believing that her choice
was deliberate, not originating from an external source''---recall the
discussion about the \textit{objet petit a}. As a consequence, a
subject can never learn from her own actions; rather, she only learns from their
effects.

\subsection{Connection to Extensive-Form Games}\label{sec:game-theory}

Before concluding this section, we briefly review the relation between the
causal interventions derived above and their connection to extensive-form
games. For the sake of brevity, here I will adopt an informal exposition 
to elucidate the connection of concern, referring the reader to the original
text by \citet[ch.~7]{Neumann1944} or the modern text by 
\citet[chs.~2, 6, and~11]{Osborne1999} for a formal description.

In game theory, an \textit{extensive-form game} is a specification of a game
between two or more players that can capture a variety of aspects of the game
mechanics, such as the order and possible moves of the players, the information available 
to them under each situation, their pay-offs, and chance moves. The representation 
of such a game consists in a rooted game tree, where each node is assigned to
a player, and where the edges leaving a node represent the possible moves that
the corresponding player can take in that node.

\begin{figure}[htbp]
\centering %
\includegraphics[width=0.8\columnwidth]{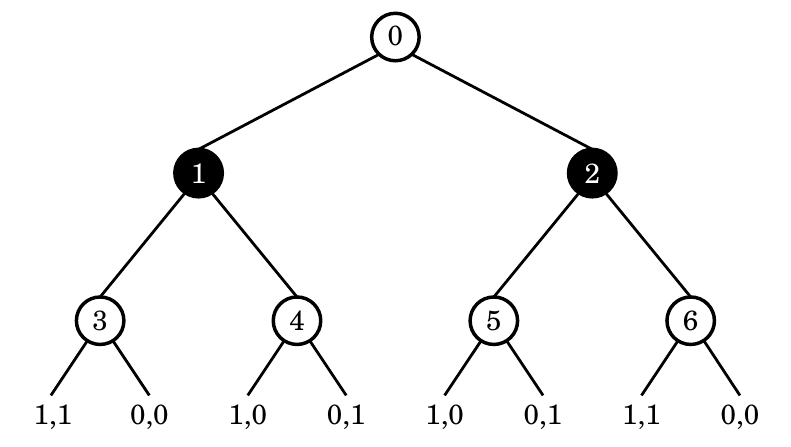}
\caption{An extensive-form game with perfect information.}
\label{fig:game-perfect} %
\end{figure}

Fig.~\ref{fig:game-perfect} illustrates a two-player game with three steps.
The internal nodes are coloured according to the player that takes that move, 
and each terminal node is labelled with the pay-off for Black and White
respectively. Here, notice that the best strategy for Black consists in choosing
\textit{left} when in node~1, and \textit{right} when in node~2, as otherwise she
would walk away empty-handed, assuming that White makes moves to maximise his 
own pay-offs.

\begin{figure}[htbp]
\centering %
\includegraphics[width=0.8\columnwidth]{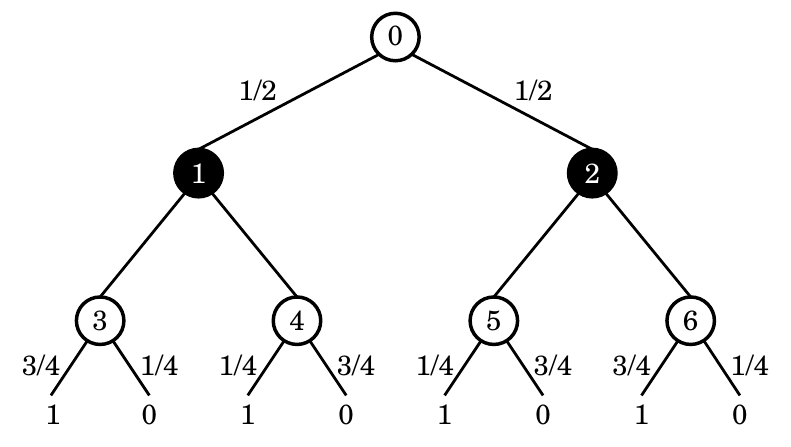}
\caption{Replacing one player with chance moves.}
\label{fig:game-perfect-chance} %
\end{figure}

We can bring chance elements into the game by replacing one of the players
with stochastic moves. In this case, White is substituted with a stochastic
strategy that is reminiscent of the original optimal strategy (\textit{i.e.}\ White
now chooses a suboptimal move with probability $1/4$), like shown in 
Fig.~\ref{fig:game-perfect-chance}. Notice that once the strategy is
settled, we can drop the pay-offs for White from the description of the game.
The new game with chance moves for White has the same optimal strategy for 
Black as the previous game.

The two previous games are known as games with \textit{perfect information},
because players know at all times what moves were taken previously in the game.
To model games where players do not see one or more of the previous moves,
von~Neumann introduced the concept of an \textit{information set}. An information
set consists of a set of nodes belonging to the same player, with the semantics
that if any of them is reached during the game, then the corresponding player
cannot distinguish between its members to make the next move. Hence, the general description
of an extensive-form game requires specifying an partition of each player's decision
nodes into information sets, and perfect information games are special in that each
player's partition is just a collection of singletons.

\begin{figure}[htbp]
\centering %
\includegraphics[width=0.8\columnwidth]{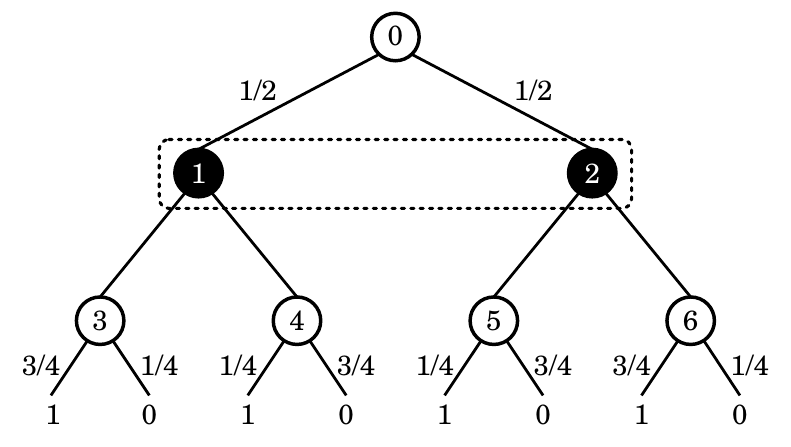}
\caption{Information sets.}
\label{fig:game-imperfect-chance} %
\end{figure}

Information sets are typically drawn as a loop around its member nodes,
and singleton sets are omitted for brevity. Fig.~\ref{fig:game-imperfect-chance}
shows the new game obtained from lumping together nodes~1 \&~2. As a result, the
previous optimal strategy for Black in not valid any more, since it prescribes
different moves for the two nodes in the set. In contrast, the new optimal strategy for
Black must take into account that either node is reached with the same 
probability. Indeed, under this constraint on Black's knowledge, it is easily 
seen that both possible actions are rendered equally good.

It hard to overstate Von~Neumann's achievement. Clearly, information sets
play the r\^{o}le of restricting the game dynamics under causal constraints, \text{i.e.}\ 
\emph{players' moves are interventions}. More precisely, once players pick their 
strategies, the resulting sequential distribution over moves can be thought 
of as the result of applying the type of causal interventions discussed 
in the preceding parts of this section (see Fig.~\ref{fig:game-intervention} for
a comparison with the associated causal DAG). 

\begin{figure}[htbp]
\centering %
\includegraphics[width=\columnwidth]{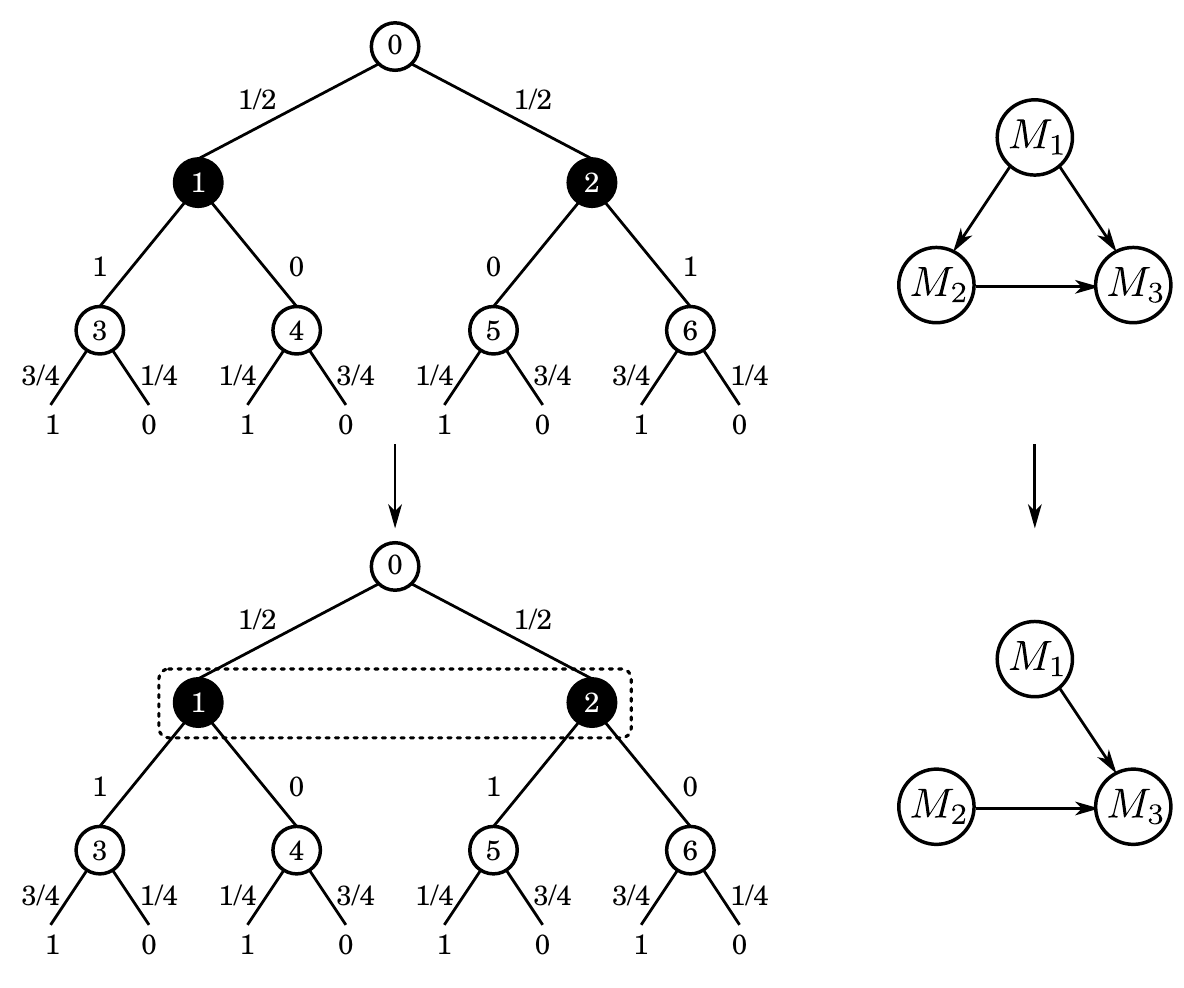}
\caption{The causal intervention is highlighted when comparing the game dynamics
before and after introducing the information sets of Black.}
\label{fig:game-intervention} %
\end{figure}

\section{The Abstract Subject}\label{sec:abstract-subject}

The aim of this section is to present an abstract model of the subject.
In particular, I have dedicated much effort into elucidating the links 
to measure-theoretic probability, which currently holds the status of
providing the standard foundations for abstract probability theory.

\subsection{Realisations and Causal Spaces}

First we introduce a structure that models the states of realisation of a random
experiment. 

\begin{definition}[Realisation]
A set $\set{R}$ of non-empty subsets of $\Omega$ is called a \emph{set of
realisations} iff
\begin{enumerate}
    \item[R1.] \emph{the sure event is a realisation:} \\
      $\Omega \in \set{R}$;
    \item[R2.] \emph{realisations form a tree:}\\
      for each distinct $U, V \in \set{R}$, either $U \cap V = \varnothing$
      or $U \subset V$ or $V \subset U$;
    \item[R3.] \emph{the tree is complete:}\\
      for each $U, V \in \set{R}$ where $V \subset U$, there exists a
      sequence $(V_n)_{n \in \nats}$ in $\set{R}$ such that $U \setminus V
      = \bigcup_n V_n$.
    \item[R4.] \emph{every branch has a starting and an end point:}\\
      let $(V_n)_{n \in \nats} \in \set{R}$ be such that $V_n \uparrow V$ or
      $V_n \downarrow V$. Then, $V \in \set{R}$.
\end{enumerate}
A member $U \in \set{R}$ is called a \emph{realisation} or a \emph{realisable
event}. Given two realisations $U, V \in \set{R}$, we say that $U$
\emph{precedes} $V$ iff $U \supset V$. Given two subsets $\set{U}, \set{V}
\subset \set{R}$, we say that $\set{U}$ \emph{precedes} $\set{V}$ iff for every
$V \in \set{V}$, there exists an element $U \in \set{U}$ such that $U$ precedes
$V$. Analogously, we also say that $V$ \emph{follows} $U$ iff $U$ precedes $V$,
and that $\set{V}$ \emph{follows} $\set{U}$ iff $\set{U}$ precedes $\set{V}$.
Finally, two realisations $U, V \in \set{R}$ that neither precede nor follow
each other are said to be \emph{incomparable}.
\end{definition}

From axioms~R1--R3, it is clearly seen that a set of realisations is
essentially a tree of nested subsets of the sample space, rooted at the sample
space. Axiom~R4 includes the upper and lower limits of realisation sequences,
thus constituting a formalisation of the fourth postulate causal reasoning. One
important difference to standard $\sigma$-algebras is that the complement is, in
general, \emph{not} in the algebra, the only exception being the impossible
realisation.

An immediate consequence of this definition is that the set of realisations
forms a partial order among its members. The partial order is the fundamental
requirement for modelling causal dependencies.

\begin{proposition}\label{prop:partial-order}[Partial Order]
A set of realisations $\set{R}$ endowed with the set inclusion $\subset$ forms
a partial order.
\end{proposition}
\begin{proof}
Trivial, because it is inherited from $(\power(\Omega), \subset)$: it is
reflexive, since for each $U \in \set{R}$, $U \subset U$; it is antisymmetric,
since for each $U, V \in \set{R}$, if $U \subset V$ and $V \subset U$ then $U =
V$; and it is transitive, because for all $U, V, W \in \set{R}$, if $W \subset
V$ and $V \subset U$ then $W \subset U$.
\end{proof}

The intuition here is that ``$U \supset V$'' corresponds to the intuitive notion
of ``$V$ depends causally on $U$'', \textit{i.e.}\ the veracity of $V$ can only be
determined insofar $U$ is known to have obtained; and ``$U \cap V =
\varnothing$'' means that ``$V$ and $U$ are causally independent''.

A set of realisations can be visualised as a tree of nested sets. For instance, 
Fig.~\ref{fig:realisation-set} is a possible set of realisations for the experiment
in Fig.~\ref{fig:experiment}. Here, the sure event~$\Omega$ at the root is partitioned
recursively into branches until reaching the leaves representing the termination of the experiment.

\begin{figure}[htbp]
\centering %
\includegraphics[width=0.8\columnwidth]{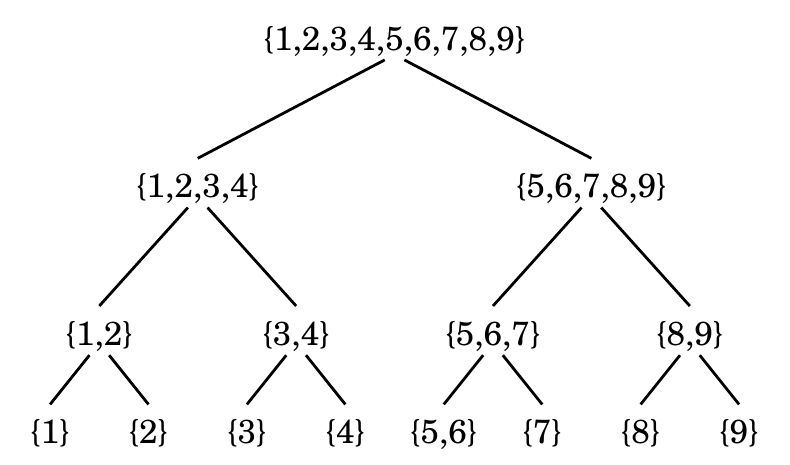}
\caption{A realisation set for the experiment.}
\label{fig:realisation-set} %
\end{figure}

Next we define an important class of events of the experiment, namely those
that can be thought of as a union of causal histories that are potentially
incompatible.

\begin{definition}[Representation]
A subset $A \subset \Omega$ is said to have a \emph{representation} in
$\set{R}$ iff there exists a sequence $(A_n)_{n \in \nats}$ in $\set{R}$ such
that $\bigcup_n A_n = A$. The \emph{set of representable events} $r(\set{R})$ is
the collection of all the subsets of $\Omega$ that have a representation in~$\set{R}$.
\end{definition}

For instance, consider the subsets
\[
A_1 = \{2,5,6\}, \quad
A_2 = \{3,4\} \quad \text{and} \quad
A_3 = \{4,5\}.
\]
$A_1$ has a unique representation given by $\mathcal{A}_1 
= \{\{ 2 \}, \{ 5,6 \}\}$; $A_2$ has two representations, namely $\mathcal{A}_2
= \{ \{3\}, \{4\} \}$ and $\mathcal{A}'_2 = \{ \{3,4\} \}$; and $A_3$ does not
have a representation.

It turns out that the set of representable events has a fundamental property: it
coincides with the $\sigma$-algebra generated by $\set{R}$. This means that
every event of the experiment can be thought of as corresponding to a
collection of possibly mutually exclusive realisations.

\begin{theorem}[Representation]\label{thm:representation}
Let $\set{R}$ be a set of realisations, let $r(\set{R})$ be the set of
representable events in $\set{R}$ and let $\sigma(\set{R})$ be the
$\sigma$-algebra generated by~$\set{R}$. Then, $r(\set{R}) = \sigma(\set{R})$.
\end{theorem}
\begin{proof}
\emph{Case $r(\set{R}) \subset \sigma(\set{R})$:} This follows directly from the
definition of a $\sigma$-algebra. \emph{Case $r(\set{R}) \supset
\sigma(\set{R})$:} We prove this by induction. For the base case, let $(V_n)_{n
\nats}$ be a sequence in $\set{R}$. Then, $V = \bigcup_n V_n \in
\sigma(\set{R})$ has a representation in $\set{R}$. Furthermore, $V \in \set{R}$
implies that there exists $(V_n)_{n \in \nats}$ in $\set{R}$ such that $\Omega
\setminus V = \bigcup_n V_n$ (Axiom~R3). Hence, $V^c \in \sigma(\set{R})$ too
has a representation in $\set{R}$. For the induction case, assume we have a
sequence $(A_n)_{n \in \nats}$ with representations $(\set{A}_n)_{n \in \nats}$
respectively, where $\set{A}_n = (V_{n,m})_{m \in \nats}$ for each $n \in
\nats$. Then,
\[
  \bigcup_n A_n
  = \bigcup_n \bigcup_m V_{n,m}
  = \bigcup_l V_l,
\]
where $l \in \nats$ is a diagonal enumeration of the $(n,m) \in \nats \times
\nats$. Obviously, $(V_l)_{l \in \nats}$ is a representation for $\bigcup_n A_n
\in \sigma(\set{R})$. Now, assume that $A \in \sigma(\set{A})$ has a
representation $(A_n)_{n \in \nats}$. Then,
\[
  A^c
  = \biggl( \bigcup_n A_n \biggr)^c
  = \bigcap_n A_n^c.
\]
Since the $A_n$ are in $\set{R}$, their complements $A_n^c$ have
representations $(V_{n,m})_{m \in \nats}$. Hence,
\[
  A^c
  = \bigcap_n A_n^c
  = \bigcap_n \bigcup_m V_{n,m}
  = \bigcup_{f:\nats \rightarrow \nats} \bigcap_n V_{n,f(n)},
\]
where the last equality holds due to the extensionality property of sets. More
specifically, for $\omega \in \Omega$ to be a member of the l.h.s., there must
be an $m$ for each $n$ such that $\omega \in V_{n,m}$. This is true in
particular for the map $f$ that chooses the smallest $m$ for each $n$. Hence,
$\omega$ is a member of the r.h.s. Now, consider an element $\omega \in \Omega$
that is not in the l.h.s. Then, there exists some $n$ such that $\omega \notin
V_{n,m}$ for all $m$. Since, for this particular $n$, this is false for any
choice of $f$ in the r.h.s., $\omega$ is not a member of the r.h.s., which
proves the equality. Finally, since intersections of members $V_{n,m}$ of
$\set{R}$ are either equal to $\varnothing$ or equal to a member $V_l$ of
$\set{R}$ (Axiom~R2), one has $A^c = \bigcup_l V_l$ for some $(V_l)_{l \in
\nats}$, which is a representation of $A^c$.
\end{proof}

Having defined the basic structure of realisations, we now place probabilities
on them. However, rather than working with the standard (unconditional)
probability measure~$\prob$ that is customary in measure-theory, here---as
is also the case in Bayesian probability theory \citep{Cox1961,Jaynes2003}---it
is much more natural to work directly with a conditional probability measure
$\prob(\cdot|\cdot)$. One way to establish the connection to the standard
measure-theoretic view consists in thinking of the conditional probability
measure as a function such that $\prob(A|\Omega) := \prob(A)$ and 
$\prob(A|U) := \prob(A \cap U) / \prob(U)$ whenever $U \in \set{R}$ is such that
$\prob(U) > 0$. Henceforth, we will drop the qualifier ``conditional'', and
just talk about the ``probability measure'' $P(\cdot|\cdot)$. 

\begin{definition}[Causal Measure]
Given a set of realisations $\set{R}$, a \emph{causal measure} is a binary set
function $\prob(\cdot|\cdot): \set{R} \times \set{R} \rightarrow
[0,1]$, such that
\begin{enumerate}
    \item[C1.] \emph{the past is certain:}\\
    For any $V, U \in \set{R}$, if $U$ precedes $V$, then
    \[ \prob(U|V) = 1; \]
    \item[C2.] \emph{incomparable realisations are impossible:}\\
    For any incomparable $V, U \in \set{R}$,
    \[ \prob(V|U) = 0; \]
    \item[C3.] \emph{sum-rule:}\\
    For any $U \in \set{R}$ and any disjoint sequence $(V_n)_{n \in \nats}$ such
    that $V_n$ follows $U$ for all $n \in \nats$ and $\bigcup_{n \in \nats} V_n
    = U$,
    \[ \sum_{n \in \nats} \prob(V_n|U) = 1; \]
    \item[C4.] \emph{product-rule:}\\
    For any $U, V, W \in \set{R}$ such that $W$ follows $V$ and $V$ follows $U$,
    \[ \prob(W|U) = \prob(W|V) \cdot \prob(V|U). \]
\end{enumerate}
\end{definition}

Thus, a causal measure is defined only over $\set{R} \times \set{R}$,
providing a supporting skeleton for the construction of a full-fledged
probability measure extending over the entire $\sigma$-algebra.
A simple way of visually representing a causal measure is by indicating 
the transition probabilities in the corresponding tree diagram, 
as illustrated in Fig.~\ref{fig:causal-space}. In the figure, the
sets have been replaced with labels, \textit{e.g.}\ $S_0 = \Omega$ and $S_4 = \{3,4\}$. 

\begin{definition}[Compatible Probability Measure]
Given a causal measure $\prob$ over a set of realisations $\set{R}$, a
probability measure $\prob'(\cdot|\cdot): \sigma(\set{R}) 
\times \sigma(\set{R}) \rightarrow [0,1]$ is said to be \emph{compatible}
with $\prob$ iff $\prob' = \prob$ on $\set{R} \times \set{R}$.
\end{definition}

\begin{figure}[htbp]
\centering %
\includegraphics[width=0.8\columnwidth]{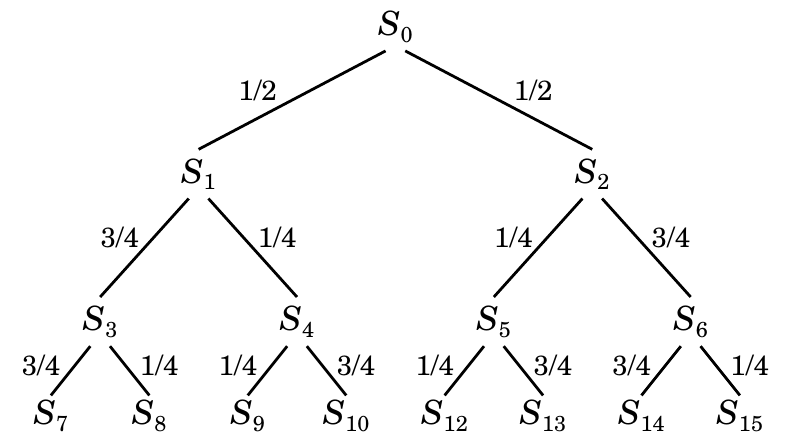}
\caption{A causal space for the experiment.}
\label{fig:causal-space} %
\end{figure}

It turns out that the causal measure almost completely determines its
compatible probability measures, the exception being the probabilities
conditioned on events that are not realisations. To show this,
we first introduce a definition.

\begin{definition}[$\set{R}_U$, $\Sigma_U$]
Let $\set{R}$ be a set of realisations. For any given $U$, define
$\set{R}_U := U \cap \set{R}$ and $\Sigma_U := U \cap \sigma(\set{R})$.
\end{definition}

Observe that $\set{R}_U$ is a set of realisations based on $U$ as the
sample space. Furthermore, it is well-known \citep[Chapter~1.2]{Ash1999} that
\[
  \Sigma_U = U \cap \sigma(\set{R}) = \sigma_U(U \cap \set{R})
\]
where $\sigma_U(U \cap \set{R})$ is the $\sigma$-algebra generated by subsets
of $U$, \textit{i.e.}\ where $U$ rather than $\Omega$ is taken as the sample space. The
aforementioned uniqueness result follows.

\begin{proposition}\label{prop:uniqueness-causal-measure}
Let $\prob_1$ and $\prob_2$ be two probability measures that are compatible with
a causal measure $\prob$ over $\set{R}$. Then, for each $U \in \set{R}$, $V \in
\Sigma_U$, $\prob_1(V|U) = \prob_2(V|U)$.
\end{proposition}
\begin{proof}
First we note that each $\set{R}_U$ is a $\pi$-system, \textit{i.e.}\  a family of
subsets of $U$ that is stable under finite intersection: $U, V \in \set{R}_U$
implies $U \cap V \in \set{R}_U$. This is because, for any $U \in
\set{R}_U$, $U \cap U \in \set{R}_U$; and for all distinct $U, V \in \set{R}_U$,
either $U \cap V = \varnothing$ or $U \subset V$ or $V \subset U$ implies that
either $U \cap V = \varnothing = U \setminus U$ or $U \cap V = U$ or $V \cap U =
V$, which are all members of $\set{R}_U$.

Next we prove that for each $U \in \set{R}$, $V \in \Sigma_U$,
$\prob_1(V|U) = \prob_2(V|U)$. Lemma~1.6. in \cite{Williams1991} states that,
if two probability measures agree on a $\pi$-system, then they also agree on the
$\sigma$-algebra generated by the $\pi$-system. Pick any $U \in \set{R}$.
Applying the lemma, we conclude that for all $V \in \Sigma_U$, $\prob_1(V|U) =
\prob_2(V|U)$. Since $U \in \set{R}$ is arbitrary, the statement of the
proposition is proven.
\end{proof}

Given the previous definition, we are ready to define our main object: the
causal space. A causal space, like a standard probability space,
serves the purpose of characterising a random experiment, but with the
important difference that it also contains information about the causal
dependencies among the events of the experiment.

\begin{definition}[Causal Space]
A \emph{causal space} is a tuple $C = (\Omega, \set{R}, \prob)$, where $\Omega$
is a set of outcomes, $\set{R}$ is a set of realisations on $\Omega$, and
$\prob$ is a causal measure over $\set{R}$.
\end{definition}

Intuitively, it is clear that a causal space contains enough information to
characterise probability spaces that represent the same experiment. These
probability spaces are defined as follows.

\begin{definition}[Compatible Probability Space]
Given a causal space $C = (\Omega, \set{R}, \prob)$, a probability space
$S = (\Omega, \sfield, \prob')$ is said to be compatible with $C$ if $\sfield
= \sigma(\set{R})$ and $\prob'$ is \emph{compatible} with $\prob$.
\end{definition}

An immediate consequence of the previous results is that compatible probability
spaces are essentially unique.

\begin{corollary}
Let $S_1 = (\Omega, \sfield_1, \prob_1)$ and $S_2 =
(\Omega, \sfield_2, \prob_2)$ be two probability spaces compatible with a
given causal space $C = (\Omega, \set{R}, \prob)$. Then,
\begin{enumerate}
 \item their $\sigma$-algebras are equal, \textit{i.e.}\ $\sfield_1 = \sfield_2$;
 \item and their probability measures are equal on any condition $U \in \set{R}$
  and $\sigma$-algebras $\Sigma_U$, \textit{i.e.}\ for any $U \in \set{R}$, $V \in
  \Sigma_U$, $\prob_1(V|U) = \prob_2(V|U)$.
\end{enumerate}
\end{corollary}

Importantly though, one cannot derive a unique causal space from a probability
space; that is to say, given a probability space, there is in general more than
one causal space that can give rise to it. Crucially, these causal spaces can
differ in the causal dependencies they enforce on the events of the
experiment, thus representing incompatible causal realisations.

\subsection{Causal Interventions}

The additional causal information contained in causal spaces serve the
purpose of characterising cause-effect relations (\textit{e.g.}\  functional dependencies)
and the effects of interventions of a random experiment. Interventions modify 
realisation processes in order to steer the outcome of a random experiment 
into a desired direction.

We begin this subsection with the formalisation of a sub-process as a sequence
of realisations of the experiment, that is, as a realisation \emph{interval}.

\begin{definition}[Interval]
Let $U, V \in \set{R}$. Define 
\[
  \set{I} := \{ W \in \set{R}: U \supset W \text{ and } W \supset V \}.
\]
Then, based on $\set{I}$, define:
\begin{align*}
 [U,V]_\set{R} &:= \set{I},
  && \text{(closed interval)}\\
 (U,V]_\set{R} &:= \set{I} \setminus \{U\},
  && \text{(right-closed interval)}\\
 [U,V)_\set{R} &:= \set{I} \setminus \{V\},
  && \text{(left-closed interval)}\\
 (U,V)_\set{R} &:= \set{I} \setminus \{U, V\}.
  && \text{(open interval)}
\end{align*}
\end{definition}

Obviously, for all $U, V \in \set{R}$,
\[
  [U, V]_\set{R} \neq \varnothing \qquad \Leftrightarrow \qquad U \supset V,
\]
so that non-empty intervals necessarily have a causal direction. Although
the previous definition covers open, half-open, and closed intervals, we will
see further down that only closed intervals play an important r\^{o}le
in the context of interventions.

For example, in Fig.~\ref{fig:causal-space} we have:
\begin{align*}
  [S_0,S_{12}]_\set{R} &= \{S_0, S_2, S_5, S_{12}\} \\
  [S_0,S_{12})_\set{R} &= \{S_0, S_2, S_5\} \\
  [S_{12},S_0]_\set{R} &= \varnothing.
\end{align*}

In a random experiment, two sub-processes with the same initial conditions
can lead to two different outcomes. Next, I define a \emph{bifurcation} and a
\emph{discriminant}, the former corresponding to the exact moment when these
two processes separate from each other and the latter to the instant
right afterwards---that is, the instant that unambiguously determines the
start of a new causal course. Notice that in what follows, I will drop the
subscript $\set{R}$ when it is clear from the context.

\begin{definition}[Bifurcations \& Discriminants]
Let $\set{R}$ be a set of realisations, and let $\set{I}_1 = [U, V_1]$ and
$\set{I}_2 = [U, V_2]$ be two closed intervals in $\set{R}$ with same initial
starting point $U$ and non-overlapping endpoints $V_1 \cap V_2 = \varnothing$. A
member $\lambda \in \set{R}$ is said to be a \emph{bifurcation} of $\set{I}_1$
and $\set{I}_2$ iff $[U, \lambda] = \set{I}_1 \cap \set{I}_2$. A member $\xi
\in \set{R}$ is said to be a \emph{discriminant} of $\set{I}_1$ from
$\set{I}_2$ iff $\set{I}_1 \setminus \set{I}_2 = [\xi, V_1]$.
\end{definition}

For instance, relative to the causal space in~Fig.~\ref{fig:causal-space}, consider
the intervals
\[
  [S_0, S_7]_\set{R} \quad \text{and} \quad [S_0, S_9]_\set{R}.
\]
Then, the bifurcation is $S_1$, because
\[
  [S_0, S_7]_\set{R} \cap [S_0, S_9]_\set{R} = [S_0, S_1]_\set{R};
\]
and their discriminants are $S_3$ and $S_4$ respectively, because
\begin{align*}
  [S_0, S_7]_\set{R} &\setminus [S_0, S_9]_\set{R} = [S_3, S_7]_\set{R}
  \quad \text{and} \\
  [S_0, S_9]_\set{R} &\setminus [S_0, S_7]_\set{R} = [S_4, S_9]_\set{R}.
\end{align*}

In principle, bifurcations and discriminants might not exist; or if they exist,
they might not be unique. The following lemma disproves this possibility by
showing that bifurcations and discriminants always exist and are unique.

\begin{lemma}\label{lem:bif}
Let $\set{R}$ be a set of realisations, and let $\set{I}_1 = [U, V_1]$ and
$\set{I}_2 = [U, V_2]$ be two closed intervals in $\set{R}$ with same initial
starting point $U$ and non-overlapping endpoints $V_1 \cap V_2 = \varnothing$.
Then, there exists
\begin{enumerate}
 \item[a)] a unique bifurcation of $\set{I}_1$ and $\set{I}_2$;
 \item[b)] a unique discriminant of $\set{I}_1$ from $\set{I}_2$;
 \item[c)] and a unique discriminant of $\set{I}_2$ from $\set{I}_1$.
\end{enumerate}
\end{lemma}
\begin{proof}
First, we observe that there cannot be any $V \in \set{I}_1$ such that $V \cap
V_1 = \varnothing$. For if this was true, then we would have that for all $W \in
[V, V_1]$, $W \cap V_1 = \varnothing$, which would lead to a contradiction since
we know that for $W = V_1$, $W \cap V_1 \neq \varnothing$. Repeating the same
argument for $\set{I}_2$, we also conclude that there cannot be any $V
\in \set{I}_2$ such that $V \cap V_2 = \varnothing$.

Second, using a similar argument as above, if $V \in \set{I}_1$ is such that $V
\cap V_2 = \varnothing$, then for all $W \in [V, V_1]$, $W \cap V_2 =
\varnothing$; and if $V \in \set{I}_1$ is such that $V \cap V_2 \neq
\varnothing$, then for all $W \in [\Omega, V]$, $W \cap V \neq \varnothing$.
This leads us to conclude that $\set{I}_1$ can be partitioned into
\begin{align*}
  [U,R_1) &:= \{W \in \set{I}_1: W \cap V_2 \neq \varnothing\} \\
  \text{and} \quad
  (S_1,V_1] &:= \{W \in \set{I}_1: W \cap V_2 = \varnothing\},
\end{align*}
for some $R_1, S_1 \subset \Omega$, that is to say, 
where $\set{I}_1 = [U,R_1) \cup (S_1,V_1]$ and
$[U,R_1) \cap (S_1,V_1] = \varnothing$. But, due to axiom~R4, both
intervals must be closed. Hence, in particular, it is true that $[\Omega, R_1)
= [\Omega, \lambda_1]$ for some $\lambda_1 \in \set{I}_1$. Similarly,
$\set{I}_2$ can be partitioned into
\begin{align*}
  [U,R_2) &:= \{W \in \set{I}_2: W \cap V_1 \neq \varnothing\} \\
  \text{and} \quad
  (S_2,V_2] &:= \{W \in \set{I}_2: W \cap V_1 = \varnothing\},
\end{align*}
and again, $[U,R_2) = [U,\lambda_2]$ for some $\lambda_2 \in \set{I}_2$.
Now, if $W \in \set{I}_1 \cup \set{I}_2$, then $W \in [U,\lambda_1]
\Leftrightarrow W \in [U,\lambda_2]$. Hence, $\lambda_1 = \lambda_2$
is unique and is a bifurcation, proving part (a). For parts (b) and (c), 
we note that $(R_1,V_1] = [\xi_1,V_1]$ and $(R_2,V_2] = [\xi_2,V_2]$
for some $\xi_1 \in \set{I}_1 \setminus \set{I}_2$ and $\xi_2 \in \set{I}_2
\setminus \set{I}_1$ respectively due to axiom~R4. But since $\set{I}_1
\setminus \set{I}_2 = \set{I}_1 \setminus [U,\lambda_1] = [\xi_1,V_1]$, and
similarly $\set{I}_2 \setminus \set{I}_1 = [\xi_2,V_2]$, the members $\xi_1$ and
$\xi_2$ are the desired discriminants, and they are unique.
\end{proof}

The important consequence of this lemma is that there is always a pair of
closed intervals $[\lambda, \xi_1]$ and $[\lambda, \xi_2]$ that \emph{precisely}
capture the sub-process, or \emph{mechanism}, during which a realisation can
split into two mutually exclusive causal branches.

To intervene a random experiment in order to give rise to a particular event
$A$, we first need to identify all the sub-processes that can split the course
of the realisation into intervals leading to $A$ or its negation $A^c$. These
sub-processes will start and end at instants that will be called
$A$-bifurcations and $A$-discriminants respectively. Again, a lemma guarantees
that these sub-processes exist and are unique.

\begin{definition}[$A$-Bifurcations, $A$-Discriminants]
Let $\set{R}$ be a set of realisations, and let $A \in \sigma(\set{R})$ be a
member of the generated $\sigma$-algebra of $\set{R}$.
\begin{enumerate}
  \item A member $\lambda \in \set{R}$ is said to be an \emph{$A$-bifurcation}
  iff it is a bifurcation of two intervals $[\Omega, V_A]$ and $[\Omega,
  V_{A^c}]$ with endpoints $V_{A}$ and $V_{A^c}$ in some representations of $A$
  and $A^c$ respectively. The \emph{set of $A$-bifurcations} is the subset
  $\lambda(A) \subset \set{R}$ of all $A$-bifurcations.
  \item Let $\lambda \in \lambda(A)$ be an $A$-bifurcation. A member
  $\xi \in \set{R}$ is said to be  an \emph{$A$-discriminant for $\lambda$} iff
  there exists $V_A$ in a representation of $A$ such that $[\xi, V_A]
  = [\Omega, V_A] \setminus [\Omega, \lambda]$. The \emph{set of
  $A$-discriminants for $\lambda$} is denoted as $\xi(\lambda)$.
\end{enumerate}
\end{definition}

Figure~\ref{fig:bifurcations} illustrates the set of $A$-bifurcations
for $A$ defined as
\[
  A = S_7 \cup S_9 \cup S_{12} \cup S_{13} = \{1, 3, 4, 5, 6, 7\}.
\]
The set of $A$-bifurcations is
\[
  \lambda(A) = \{S_0, S_1, S_2, S_3, S_4\}.
\]
Each member has an associated set of $A$-discriminants. For instance,
\[
  \xi(S_0) = \{ S_1, S_2 \}
  \quad \text{and} \quad
  \xi(S_2) = \{ S_5 \}.
\]
The critical bifurcations that appear in the figure are a subset of the 
bifurcations, and they will be defined later in the text.

\begin{figure}[htbp]
\centering %
\includegraphics[width=0.8\columnwidth]{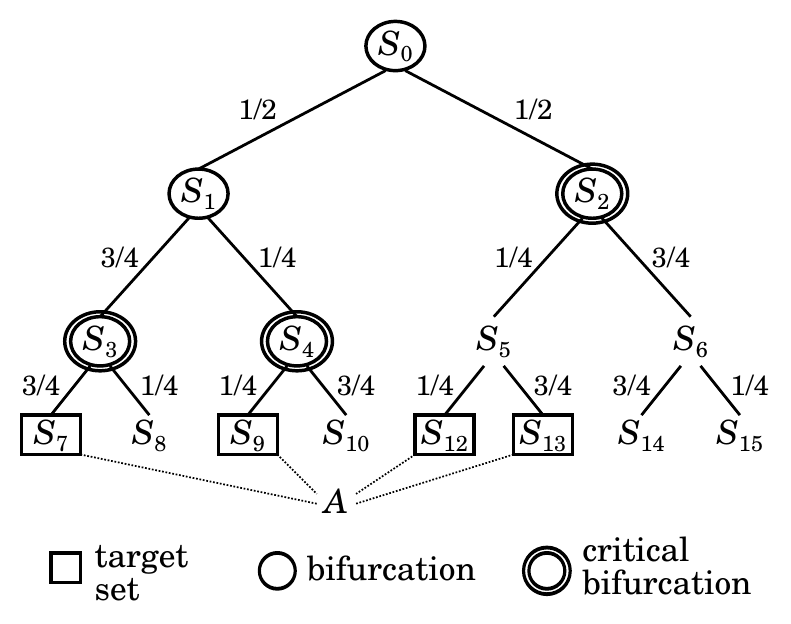}
\caption{$A$-bifurcations.}
\label{fig:bifurcations} %
\end{figure}

\begin{lemma}\label{lem:existence-bif}
Let $\set{R}$ be a set of realisations, and let $A \in \sigma(\set{R})$ be a
member of the generated $\sigma$-algebra of $\set{R}$. Then, the set of
$A$-bifurcations $\lambda(A)$ and the sets of $A$-discriminants $\xi(\lambda)$,
$\lambda \in \lambda(A)$, exist, are countable, and unique.
\end{lemma}
\begin{proof}
Let $\set{A}_1$ and $\set{A}^c_1$ be representations of $A$ and $A^c$
respectively. Consider the set $\lambda_1(A) \subset \set{R}$ of
bifurcations of $[\Omega, V_A]$ and $[\Omega, V_{A^c}]$ generated by all
pairs of endpoints $(V_A, V_{A^c}) \in \set{A} \times \set{A}^c$. Due to the
representation theorem, we know that both $\set{A}$ and
$\set{A}^c$ are countable. Therefore, $\set{A} \times \set{A}^c$
and $\lambda_1(A)$ are countable too. Now, repeat the same procedure to
construct a set $\lambda_2(A)$ of bifurcations from two representations
$\set{A}_2$ and $\set{A}^c_2$ of $A$ and $A^c$ respectively.

Let $R \in \lambda_1(A)$. Then, there exists $R_A \in \set{A}_1$ and
$R_{A^c} \in \set{A}^c_1$ such that $[\Omega, R] = [\Omega, R_A] \cap [\Omega,
R_{A^c}]$. Since $\set{A}_2$ and $\set{A}^c_2$ are representations of $A$ and
$A^c$ respectively, there must be members $S_A \in \set{A}_2$ and $S_{A^c}
\in \set{A}^c_2$ such that $R_A \cap S_A \neq \varnothing$ and $R_{A^c} \cap
S_{A^c} \neq \varnothing$. Due to axiom~R2, it must be that either $R_A \subset
S_A$ or $S_A \subset R_A$; similarly, either $R_{A^c} \subset S_{A^c}$ or
$S_{A^c} \subset R_{A^c}$. But then, $[\Omega, S_A] \cap [\Omega, S_{A^c}] =
[\Omega, R]$, implying that $R \in \lambda_2(A)$. Since $R$ is arbitrary,
$\lambda_1(A) = \lambda_2(A)$. Hence, we have proven that $\lambda(A)$
exists, is countable, and unique.

Let $\lambda \in \lambda(A)$ be an arbitrary $A$-bifurcation. Let
$\set{A}_1$ and $\set{A}_2$ be two representations of $A$. Because
$A$-representations are countable, there exists a countable number of intervals
$[\Omega, V_A]$, $V_A \in \set{A}_1$, containing $\lambda$ and an associated
discriminant. Let $\xi_1(\lambda)$ be the collection of those discriminants.
Following an argument analogous as above, it is easy to see that
$\xi_2(\lambda)$, the set of discriminants constructed from the intervals
associated to the representation $\set{A}$, must be equal to $\xi_1(\lambda)$.
Hence, for each $\lambda \in \lambda(A)$, $\xi(\lambda)$ exists, is countable
and unique.
\end{proof}

We are now ready to define interventions on a causal space. In the next
definition, an $A$-intervention is defined as the change of the causal
measure at the bifurcations such that the desired event $A$ will inevitably take
place. This is done by removing all the probability mass leading to the
undesired event $A^c$ and renormalising thereafter. 

\begin{definition}[Intervention]\label{def:intervention}
Let $\set{R}$ be a set of realisations, $\prob$ be a causal measure, and
$A$ be a member of the generated $\sigma$-algebra of $\set{R}$. A causal measure
$\prob'$ is said to be an $A$-\emph{intervention of} $\prob$ iff for all $U, V
\in \set{R}$ such that $V \cap A \neq \varnothing$,
\begin{equation}\label{eq:intervention}
  \prob'(V|U) \cdot G(U,V) = \prob(V|U),
\end{equation}
where $G(U,V)$ is the \emph{gain} of the interval $[U,V]$ defined by
\begin{equation}\label{eq:gain}
  G(U,V) := \prod_{\lambda \in \Lambda} 
    \sum_{\xi \in \xi(\lambda)} \prob(\xi|\lambda).
\end{equation}
Here, $\Lambda := [U,V] \cap \lambda(A)$ is the set of bifurcations 
in $[U,V]$, and each $\xi(\lambda)$ is the set of $A$-discriminants
of $\lambda \in \Lambda$.
\end{definition}

In the tree visualisation, an $A$-intervention can be thought of as
a reallocation of the probability mass at the $A$-bifurcations 
(see Fig.~\ref{fig:intervention}). Essentially, this is done
by first removing the probability mass from the transitions that do
not have any successor realisation in $A$ and then by renormalising
the remaining transitions, \textit{i.e.}\ the ones rooted at $A$-discriminants.

\begin{figure}[htbp]
\centering %
\includegraphics[width=0.8\columnwidth]{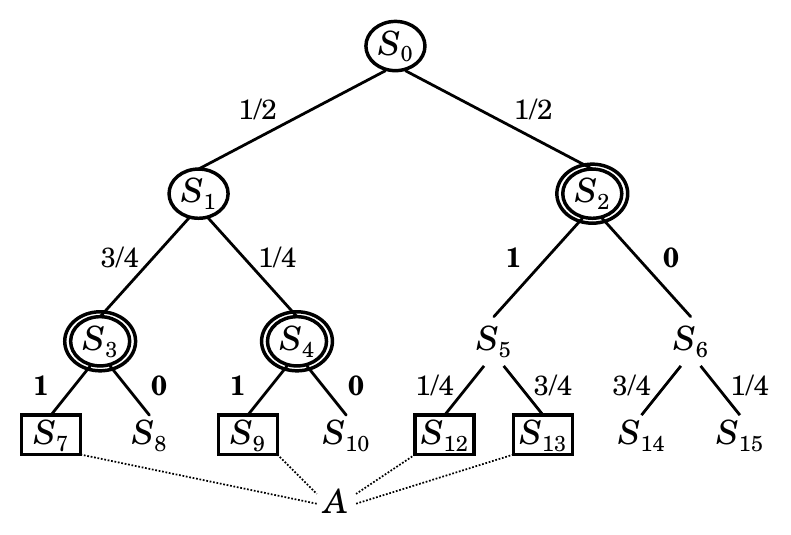}
\caption{$A$-intervention.}
\label{fig:intervention} %
\end{figure}

\begin{theorem}[Uniqueness of $A$-Interventions]\label{thm:uniqueness-int}
Let $\set{R}$ be a set of realisations, $\prob$ be a causal measure, and
$A$ be a member of the generated $\sigma$-algebra of $\set{R}$. The
$A$-intervention is unique if for each bifurcation $\lambda \in \lambda(A)$,
the corresponding $A$-discriminants are not null, \textit{i.e.}\ they are such that
\begin{equation}\label{eq:intervention-condition}
  \sum_{\xi \in \xi(\lambda)} \prob(\xi|\lambda) > 0. 
\end{equation}
\end{theorem}
\begin{proof}
Let $\prob'$ be an $A$-intervention. Because it is a causal measure,
$\prob'(V|U) = 1$ when $V \supset U$ and $\prob'(V|U) = 0$ when $V \cap U =
\varnothing$. It remains to check that $\prob'(V|U)$ is unique when $V \subset
U$. If $V \cap A \neq \varnothing$, then the definition applies. Here,
$\lambda(A)$ and the $\{\xi(\lambda)\}_{\lambda \in \lambda(A)}$ are unique due
to Lemma~\ref{lem:bif}, thus so are the $\Lambda := [U,V] \cap
\lambda(A)$ for each $U, V \in \set{R}$. Hence, we see that if
condition~\eqref{eq:intervention-condition} holds for all $\lambda \in \lambda(A)$,
then~\eqref{eq:intervention} has a unique solution for $\prob'(V|U)$.
Finally, if $V \cap A = \varnothing$ then $\prob'(V|U)$ depends on
where $[U,V]$ contains an $A$-bifurcation. If it does not, then $\prob'(V|U)
= \prob(V|U)$. If it does, then Axiom~C3 implies $\prob'(V|U) = 0$.
\end{proof}

\begin{corollary}
$A$-interventions are unique up to intervals containing only null discriminants.
In other words, given two $A$-interventions $\prob'_1$ and $\prob'_2$ let
$U,V \in \set{R}$. Then, $\prob'_1(V|U) = \prob'_2(V|U)$ if for all
$\lambda \in [U,V] \cap \lambda(A)$, there exists $\xi \in \xi(\lambda)$ such
that $\prob(\xi|\lambda) > 0$.
\end{corollary}

Finally, the next proposition shows that the intervention is indeed correct in
the sense that the desired event occurs with certainty after the intervention.

\begin{proposition}\label{prop:certainty-int}
Let $\set{R}$ be a set of realisations and let $A$ be a member of the generated
$\sigma$-algebra $\Sigma$ of $\set{R}$. Furthermore, let $\prob'$ be
probability measure compatible with an $A$-intervention of a causal measure
$\prob$. Then, $\prob'(A|\Omega) = 1$.
\end{proposition}
\begin{proof}
Let $\set{A}$ and $\set{A}^c$ be representations of $A$ and $A^c$ in $\set{R}$
respectively. Then, each $V \in \set{A}^c$ is such that $V \cap A =
\varnothing$. Hence, due to (I1),
\[
  \prob'(A^c|\Omega) = \sum_{V \in \set{A}^c} \prob'(V|\Omega) = 0.
\]
But then, since $\prob'$ is a probability measure,
\[
  \prob'(A|\Omega) = 1 - \prob'(A^c|\Omega) = 1.
\]
\end{proof}

A closer look at the definition of an $A$-intervention reveals that, while the
set of bifurcations $\lambda(A)$ contains all the logically required
bifurcations (i.e.\ all moments having a branch leading to the undesired
event $A^c$), some of them remain unaltered after the intervention. In particular,
this is always the case when the mechanisms assign zero probability to the branches
leading to~$A^c$.

\begin{definition}[Critical Bifurcations]
Let $\set{R}$ be a set of realisations, $\prob$ be a causal measure, and
$A$ be a member of the generated $\sigma$-algebra of $\set{R}$. A
bifurcation $\lambda \in \lambda(A)$ is said to be \emph{critical} iff the
corresponding $A$-discriminants are not complete, \textit{i.e.}\ they are such that
\begin{equation}\label{eq:critical-condition}
  \sum_{\xi \in \xi(\lambda)} \prob(\xi|\lambda) < 1.
\end{equation}
\end{definition}

\begin{proposition}\label{prop:critical-gain}
The gain~\eqref{eq:gain} is equal to
\begin{equation}\label{eq:critical-gain}
  G(U,V) =
  \prod_{\lambda \in \Gamma}
        \sum_{\xi \in \xi(\lambda)} \prob(\xi|\lambda)
\end{equation}
where $\Gamma$ is the set of critical bifurcations in the
interval from $U$ to $V$, and each $\xi(\lambda)$ is the set of $A$-discriminants
of $\lambda \in \Lambda$.
\end{proposition}
\begin{proof}
Partition $\Lambda$ into $\Gamma$ and 
$\overline{\Gamma} = \Lambda\setminus \Gamma$,
where $\Gamma$ contains only the critical $A$-bifurcations. Then,  
\begin{align*}
  \prod_{\lambda \in \Lambda}
    \sum_{\xi \in \xi(\lambda)} &\prob(\xi|\lambda)
  = \biggl\{ \prod_{\lambda \in \overline{\Gamma}}
    \sum_{\xi \in \xi(\lambda)} \prob(\xi|\lambda) \biggr\}
    \cdot \biggl\{ \prod_{\lambda \in \Gamma}
    \sum_{\xi \in \xi(\lambda)} \prob(\xi|\lambda) \biggr\}
  \\&= 1 \cdot
    \biggl\{ \prod_{\lambda \in \Gamma}
    \sum_{\xi \in \xi(\lambda)} \prob(\xi|\lambda) \biggr\}
  = \prod_{\lambda \in \Gamma}
    \sum_{\xi \in \xi(\lambda)} \prob(\xi|\lambda).
\end{align*}
\end{proof}

\subsection{Random Variables}

We recall the formal definition of a random variable. Let $(\Omega, \Sigma,
\prob)$ be a probability space and $(S, \Xi)$ be a measurable space.
Then an $(S,\Xi)$-valued random variable is a function $X : \Omega
\rightarrow S$ which is $(\Sigma, \Xi)$-measurable, \textit{i.e.}\ for every member $B
\in \Xi$, its preimage $X^{-1}(B) \in \Sigma$ where $X^{-1}(B) = \{\omega:
X(\omega) \in B\}$. If we have a collection $(X_\gamma: \gamma \in \Gamma)$ of
mappings $X_\gamma:\Omega \rightarrow S_\gamma$, then
\[
  \Sigma := \sigma( X_\gamma: \gamma \in \Gamma )
\]
is the smallest $\sigma$-algebra $\Sigma$ such that each $X_\gamma$ is
$\Sigma$-measurable.

The lesson we have learnt so far is that it is not enough to just specify the
$\sigma$-algebra of a random experiment in order to understand the effect of
interventions; rather, we need to endow the $\sigma$-algebra with a causal
structure. While one would expect the same to hold for random variables one
wishes to intervene, we will see that it is not necessary to explicitly model
causal dependencies among them. Instead, it is sufficient to establish a
link to some abstract causal space that is shared by all the random variables.
We begin this investigation thus with a definition of a function having causal
structure.

\begin{definition}[Realisable Function]
Let $\Omega$, $S$ be sets and $\set{R}$ and $\set{S}$ be sets of realisations
over $\Omega$ and $S$ respectively. A function $X: \Omega \rightarrow S$  is
said to be $(\set{R}, \set{S})$-realisable iff for every $B \in \set{S}$, the
preimage $X^{-1}(B)$ is a member of $\set{R}$. The following picture
illustrates this:
\[\xymatrix{
  \Omega \ar[r]^{X} & S \\
  \set{R} & \set{S} \ar[l]_{X^{-1}}
}\]
\end{definition}

The next proposition shows that realisable functions are measurable
functions, but not vice versa---as intuition immediately predicts.  

\begin{proposition}\label{prop:realisable-measurable}
Let $\set{R}$ and $\set{S}$ be two sets of realisations over sets $\Omega$ and
$S$ respectively. Let $\Sigma = \sigma(\set{R})$ and $\Xi = \sigma(\set{S})$.
If a mapping $X : \Omega \rightarrow S$ is $(\set{R}, \set{S})$-realisable then
it is also $(\Sigma, \Xi)$-measurable. However, the converse is not necessarily
true. In a diagram:
\[
\xymatrix{
  \Omega \ar[r]^{X} & S \ar @{} [dr] |{\Longrightarrow}
    & \Omega \ar[r]^{X} & S \\
  \set{R} & \set{S} \ar[l]^{X^{-1}}
    & \Sigma & \Xi \ar[l]^{X^{-1}}
}
\]
\end{proposition}
\begin{proof}
Let $B \in \Xi$. Then, Theorem~\ref{thm:representation}
tells us that there exists a representation $\set{B}$ of $B$ in $\set{S}$. Since
$X$ is $(\set{R}, \set{S})$-realisable, every member $V \in \set{B}$ has a
preimage that is in $\set{R}$, \textit{i.e.}\ $X^{-1}(V) \in \set{R}$. But
$\bigcup_{V \in \set{B}} X^{-1}(V) = X^{-1}(B)$ and $\bigcup_{V \in \set{B}}
X^{-1}(V) \in \Sigma$, hence $X$ is $(\Sigma, \Xi)$-measurable.

To disprove the converse, consider the following counterexample. Take
$\Omega = \{ \omega_1, \omega_2, \omega_3, \omega_4\}$ and $S = \{s_1,
s_2\}$. Let $\set{R} = \{ R_1, \ldots R_8 \}$, where: $R_1 = \Omega$; $R_2 =
\{\omega_1, \omega_2\}$; $R_3 = \{\omega_3, \omega_4\}$; $R_4 = \{\omega_1\}$;
$R_5 = \{\omega_2\}$; $R_6 = \{\omega_3\}$; $R_7 = \{\omega_4\}$; and $R_8 =
\varnothing$. Furthermore, let $\set{S} = \{ S_1, S_2, S_3, S_4 \}$, where:
$S_1 = S$; $S_2 = \{s_1\}$; $S_3 = \{s_2\}$; and $S_4 = \varnothing$. Let $X$
be such that $X(R_2) = S_2$ and $X(R_6) = S_3$. Observe that $S_1 = S_2
\cup S_3 \in \Xi$ and $S_4 = S_1^c \in \Xi$, and in particular $S_1, S_4 \in
\set{S}$. However, $X^{-1}(S_1) = R_2 \cup R_6$, which is obviously in
$\Sigma$ but not in $\set{R}$.
\end{proof}

Next, realisable random variables are simply defined as realisable functions
endowed with a causal measure.

\begin{definition}[Realisable Random Variable] Let $(\Omega, \set{R}, \prob)$ be
a causal space. A \emph{realisable random variable} $X$ is an $(S, \Xi)$-valued
function that is $(\set{R},\set{S})$-realisable, where $\Xi = \sigma(\set{S})$.
\end{definition}

Finally, we define the intervention of a random variable. Let $B$
be a measurable event in $\Xi$, the $\sigma$-algebra in the range
of $X$. Then, a $B$-intervention of $X$ is done in two steps: first, $B$ is
translated into its corresponding event $A$ living in the abstract
$\sigma$-algebra $\Sigma$; second, the resulting event $A$ is intervened.

\begin{definition}[Intervention of a Random Variable]
Let $(\Omega, \set{R}, \prob)$ be a causal space and let $X$ be a $(\set{R},
\set{S})$-realisable random variable. Given a set $B \in \Xi =
\sigma(\set{S})$ of the generated $\sigma$-algebra, a \emph{$B$-intervention of
the realisable random variable $X$} is a $X^{-1}(B)$-intervention of the
causal measure $\prob$.
\end{definition}

\begin{corollary}
$B$-interventions of a realisable random variable are unique up to intervals
containing only null discriminants.
\end{corollary}

This concludes our abstract model of causality. It is immediately seen that a
causal stochastic process can be characterised as a collection $(X_\gamma :
\gamma \in \Gamma)$ of $(\set{R}, \set{S}_\gamma)$-realisable random variables
$X:\Omega \rightarrow S_\gamma$ respectively defined over a shared causal space
$(\Omega, \set{R}, \prob)$. This is in perfect accordance with the theory so
far developed. 

For instance, consider a collection of four binary random variables 
$X, Y, Z$, and $U$ accommodated within the causal space $C=(\Omega, \set{R}, \prob)$
from the example shown in Fig.~\ref{fig:causal-space}. For these random variables
to be realisable, they must map elements from the sample space $\Omega$ into
$S = \{0,1\}$ such that the partition induced contains only members in the set
of realisations $\set{R}$. This is achieved by ensuring that each path from
the root to a leave contains \emph{exactly one} value assignment for each random
variable (technically, a \textit{cut} through the tree) as exemplified in 
Fig.~\ref{fig:example}. For instance, all the realisation paths must necessarily
pass either through $S_1$ or $S_2$ where $X$ assumes the value $X=0$ or $X=1$ 
respectively.

\begin{figure}[htbp]
\centering %
\includegraphics[width=0.9\columnwidth]{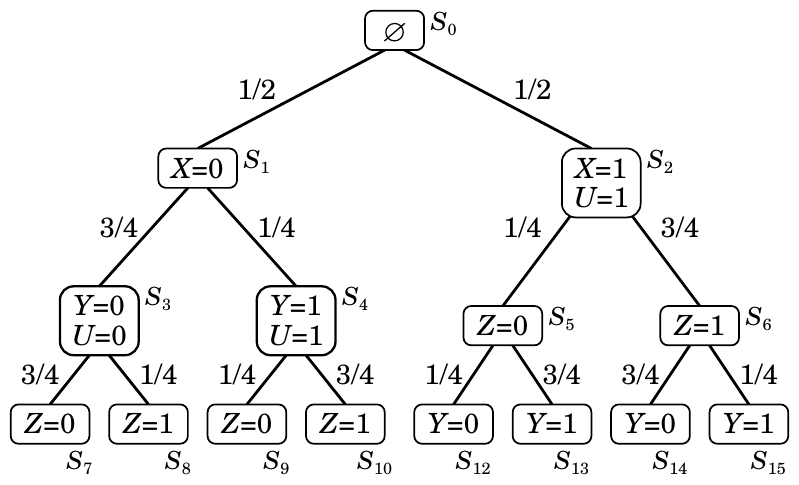}
\caption{A causal stochastic process.}
\label{fig:example} %
\end{figure}

The causal measure over $\set{R}$ extends to the random variables in the obvious
way; thus here, we have
\begin{align*}
  \prob(X=0&) = \frac{1}{2}, & \prob(X=1&) = \frac{1}{2}, \\
  \prob(Y=0&|X=0) = \frac{3}{4}, & \prob(Y=1&|X=0) = \frac{1}{4}, \\
  \prob(Z=0&|X=1) = \frac{1}{4}, & \prob(Z=1&|X=1) = \frac{3}{4}, \\
  \vdots
\end{align*}
and so on.

The example also illustrates the possibility of modelling dynamically
instantiated causal dependencies. This is seen by noting that~$Y$ 
precedes~$Z$ if $X=0$ but~$Y$ succeeds~$Z$ if $X=1$, \textit{i.e.}\ $X$
is a second-order causal dependency that controls whether $Y$ causes $Z$ 
or \textit{vice versa}. Obviously, the same idea can be applied to model
higher-order causal dependencies. 

\begin{figure}[htbp]
\centering %
\includegraphics[width=0.9\columnwidth]{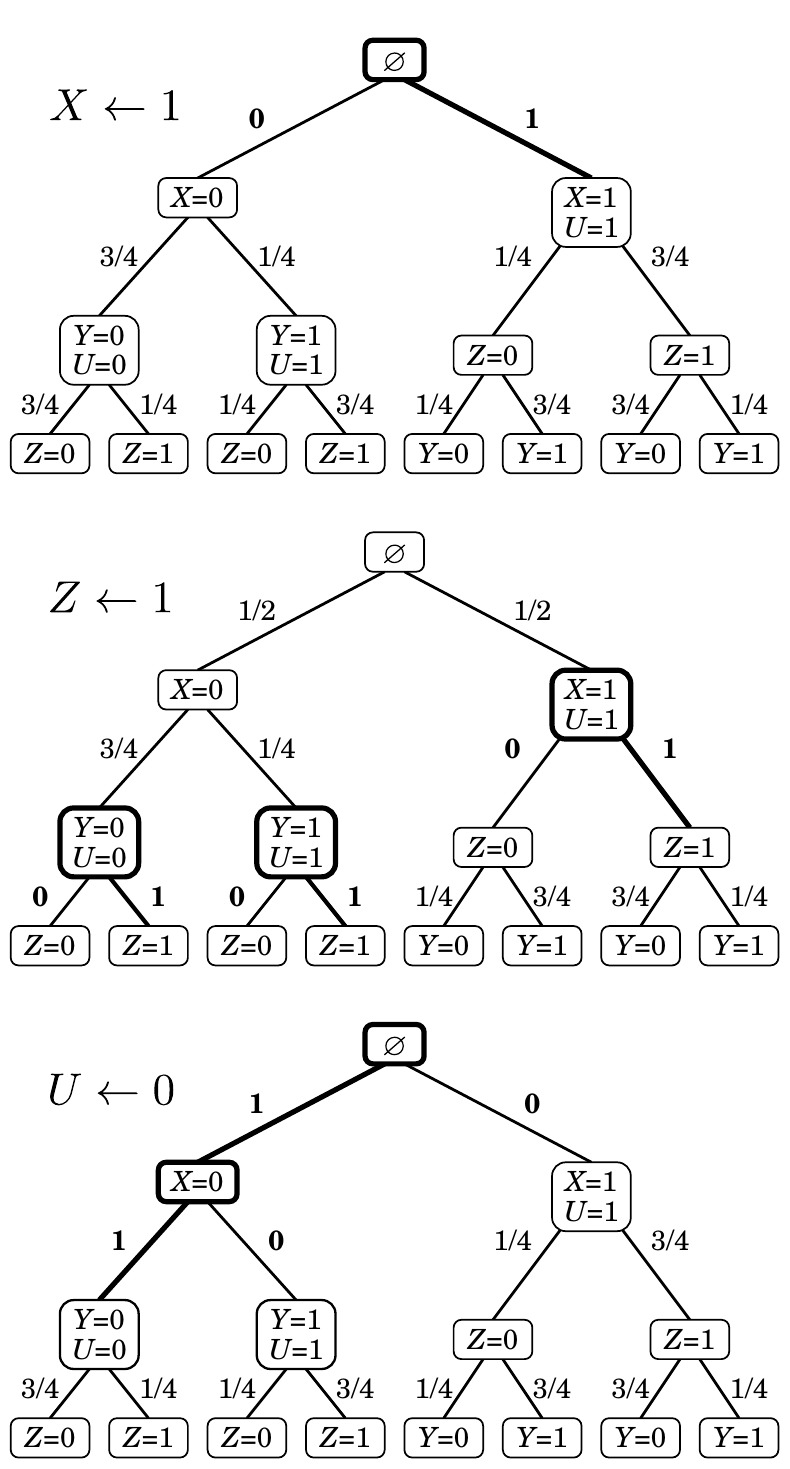}
\caption{Three example interventions.}
\label{fig:example-interventions} %
\end{figure}

Fig.~\ref{fig:example-interventions} shows three interventions, namely 
$X \leftarrow 1$, $Z \leftarrow 1$, and $U \leftarrow 0$, where the 
critical bifurcations are highlighted with thicker outlines. As mentioned
before, the intervention $X \leftarrow 1$ picks the direction of the causal
dependency of the pair $(Y, Z)$ by setting $Z$ to be the cause and $Y$ its effect.
If instead the subject is interested in finding out how to bring about 
$Z = 1$ herself, then she can do so by inspecting the intervention
$Z \leftarrow 1$. Enumerating the critical bifurcations, she can conclude 
that there are three mutually exclusive circumstances where this
can ocurr, following two different causal narratives. Finally, she
decides to pick the leftmost one through the intervention $U \leftarrow 1$
which combines changes at two critical bifurcations.

\bibliographystyle{model2-names}
\bibliography{bibliography}

\end{document}